%% file: FoSSaCS2021.tex
\documentclass[runningheads,envcountsame]{llncs}
\usepackage[T1]{fontenc}
\usepackage[utf8]{inputenc}
\usepackage{comment}
\usepackage{amsfonts,amsmath,mathrsfs}
\usepackage{enumerate}
\usepackage{tikz}
\usetikzlibrary{cd,arrows}
\tikzset{	>=stealth',
	bend angle=45,
	auto,
	baseline=(current  bounding  box.center),
}
\tikzcdset{arrow style=tikz}

\usepackage{graphicx}
\usepackage{marvosym}

\let\temp\varphi
\let\varphi\phi
\let\phi\temp

\let\temp\S
\let\S\sect

\newcommand{\bb}[1]{[\![ #1 ]\!]}
\newcommand{\zero}{\varepsilon}
\newcommand{\Cppo}{\mathbb{CPPO}}

\newcommand{\C}{\mathbb{C}}
\newcommand{\Set}{\mathbb S \mathrm{et}}
\newcommand{\Rel}{\mathbb R \mathrm{el}}
\newcommand{\Rp}{\mathbb R^+}
\newcommand{\EM}[1]{\mathbb {EM}(#1)} %Eilenberg & Moore category
\newcommand{\Kl}[1]{\mathbb K \mathrm l(#1)} %Kleisli category
\newcommand{\F}[1]{F^{#1}} %notation for free functor into EM category for a monad
\newcommand{\U}[1]{U^{#1}} %notation for forgetful functor from EM category  for a monad
\newcommand{\FK}[1]{F_{#1}} %notation for free functor into Kleisli category for a monad
\newcommand{\UK}[1]{U_{#1}} %notation for forgetful functor from Kleisli category for a monad

\newcommand{\N}{\mathbb N}
\newcommand{\natset}[1]{\underline{#1}} %notation for natural numbers seen as sets
\newcommand{\Bool}{\mathbb B \mathrm{ool}}

\renewcommand{\P}{\mathcal P} %Powerset monad
\newcommand{\Pf}{\mathcal P_f} %finite powerset
\newcommand{\pow}{\mathcal P}
\newcommand{\etaP}{\eta^\P}
\newcommand{\muP}{\mu^\P}
\newcommand{\Graph}[1]{\Gamma(#1)}
\newcommand{\pre}[1]{{#1}^{-1}}

\newcommand{\mon}[1]{\mathcal S} %semiring monad notation
\newcommand{\etaS}{\eta^\S}
\newcommand{\muS}{\mu^\S}
\DeclareMathOperator{\supp}{\mathit{supp}}

\newcommand{\convpow}{\mathcal P_c}
\newcommand{\ConvPow}[2]{\mathcal P_c^{#2} #1} %\ConvPow X a = P_c^a(X) convex powerset of X wrt a
\newcommand{\convclos}[2]{\overline{#1}^{#2}}
\newcommand{\choice}[1]{\mathfrak c(#1)} %choices of linear combinations of elements for Phi in SP(X)

\newcommand{\convpowS}{{\mathcal P_c \mathcal S}} %composite monad "big"
\newcommand{\convpowfS}{{\mathcal P_{fc} \mathcal S}} %composite monad restricted to fin gen

\newcommand{\id}[1]{\mathit{id}_{#1}}

\newcommand{\Alg}{\mathbb{A}\mathrm{lg}} %Algebras for a theory category
\newcommand{\CSL}{\mathcal{CSL}} %Theory of complete semilattices
\newcommand{\LSM}{\mathcal{LSM}}
\newcommand{\SL}{\mathcal{SL}}
\newcommand{\D}{\mathcal{D}}

\newcommand{\Sup}{\bigsqcup}

\newcommand{\card}[1]{|#1|}

\newcommand{\Sharp}[1]{{#1}^\sharp}

\spnewtheorem*{notation}{Notation}{\bfseries}{\rmfamily}

\allowdisplaybreaks %allows vertical chains of equation to split between pages if needed

\begin{document}

\title{Combining Semilattices and Semimodules\thanks{Supported by the Ministero dell'Universit\`a e della Ricerca of Italy under Grant No.\ 201784YSZ5, PRIN2017 -- ASPRA (\emph{Analysis of Program Analyses}).}}
\author{Filippo Bonchi\inst{}%\orcidID{0000-1111-2222-3333} 
\and
Alessio Santamaria%\thanks{Corresponding author}
(\Letter)
\inst{}}%\orcidID{0000-0001-7683-5221}}

\authorrunning{F.\ Bonchi and A.\ Santamaria}
\institute{Dipartimento di Informatica, Università degli Studi di Pisa, 56127 Pisa, Italy
\email{filippo.bonchi@unipi.it, alessio.santamaria@di.unipi.it}}%\\
\maketitle              % typeset the header of the contribution

\begin{abstract}
We describe the canonical weak distributive law $\delta \colon \S \P \to \P \S$ of the powerset monad $\P$ over the $S$-left-semimodule monad $\S$, for a class of semirings $S$. We show that the composition of $\P$ with $\S$ by means of such $\delta$ yields \emph{almost} the monad of convex subsets previously introduced by Jacobs: 
the only difference consists in the absence in Jacobs's monad of the empty convex set. 
We provide a handy characterisation of the canonical weak lifting of $\P$ to $\EM{\S}$ as well as an algebraic theory for the resulting composed monad. Finally, we restrict the composed monad to finitely generated convex subsets and we show that it is presented by an algebraic theory combining semimodules and semilattices with bottom, which are the algebras for the finite powerset monad $\P_f$.

\end{abstract}

\section{Introduction}
Monads play a fundamental role in different areas of computer science since they embody notions of computations~\cite{DBLP:journals/iandc/Moggi91}, like nondeterminism, side effects and exceptions. 
 Consider for instance automata theory: deterministic automata can be conveniently regarded as certain kind of coalgebras on $\Set$~\cite{DBLP:conf/concur/Rutten98}, nondeterministic automata as the same kind of coalgebras but on $\EM{\P_f}$~\cite{DBLP:journals/corr/abs-1302-1046}, and weighted automata on $\EM{\S}$~\cite{DBLP:journals/iandc/BonchiBBRS12}. Here, $\P_f$ is the finite powerset monad, modeling nondeterministic computations, while $\S$ is the monad of semimodules over a semiring $S$, modelling various sorts of quantitative aspects when varying the underlying semiring $S$. It is worth mentioning two facts: first, rather than taking coalgebras over $\EM{T}$, the category of algebras for the monad $T$, one can also consider coalgebras over $\Kl{T}$, the Kleisli category induced by $T$~\cite{hasuo_generic_2006}; second, these two approaches based on monads have lead not only to a deeper understanding of the subject, but also to effective proof techniques~\cite{DBLP:conf/concur/Bonchi0P18,DBLP:conf/csl/BonchiPPR14,erkens2020up}, algorithms~\cite{barlocco2019coalgebra,DBLP:journals/cacm/BonchiP15,van2020learning,smolka2019guarded,urbat2020automata} and logics~\cite{hasuo2015generic,DBLP:conf/popl/HasuoSC16,DBLP:conf/fossacs/KlinR15}.
 
Since compositionality is often the key to master complex structures, computer scientists devoted quite some efforts to \emph{compose monads}~\cite{DBLP:journals/mscs/VaraccaW06} or the equivalent notion of \emph{algebraic theories} \cite{DBLP:journals/tcs/HylandPP06}. Indeed, the standard approach of composing monads by means of \emph{distributive laws}~\cite{beck_distributive_1969} turned out to be somehow unsatisfactory. On the one hand, distributive laws do not exist in many relevant cases: see~\cite{DBLP:journals/entcs/KlinS18,zwart_no-go_2019} for some no-go theorems; on the other hand, proving their existence is error-prone: see~\cite{DBLP:journals/entcs/KlinS18} for a list of results that were mistakenly assuming the existence of a distributive law of the powerset monad over itself. 

Nevertheless, some sort of weakening of the notion of distributive law--e.g., distributive laws of functors over monads~\cite{DBLP:journals/tcs/Klin11}--proved to be ubiquitous in computer science: they are GSOS specifications~\cite{turi1997towards}, they are sound coinductive up-to techniques~\cite{DBLP:conf/csl/BonchiPPR14} and complete abstract domains~\cite{DBLP:conf/lics/BonchiGGP18}. 
In this paper we will exploit \emph{weak distributive laws} in the sense of~\cite{garner_vietoris_2020} that have been recently shown successful in composing the monads for nondeterminism and probability~\cite{goy_combining_2020}.

\medskip

The goal of this paper is to somehow combine the monads $\P_f$ and $\S$ mentioned above. Our interest in $\S$ relies on the wide expressiveness provided by the possibility of varying $S$: for instance by taking $S$ to be the Boolean semiring, one obtains the monad $\P_f$; by fixing $S$ to be the field of reals, coalgebras over $\EM{\S}$ turn out be linear dynamical systems~\cite{DBLP:journals/tcs/Rutten05}. 

\medskip

We proceed as follows. Rather than composing $\P_f$, we found it convenient to compose the \emph{full}, not necessarily finite, powerset monad $\P$ with $\S$. In this way we can reuse several results in~\cite{clementino_monads_2014} that provide necessary and sufficient conditions on the semiring $S$ for the existence of a canonical weak~\cite{garner_vietoris_2020} distributive law  $\delta \colon \S \P \to \P \S$. 
Our first contribution (Theorem~\ref{thm:delta for positive refinable semifields}) consists in showing that such $\delta$ has a convenient alternative characterisation, whenever the underlying semiring is a \emph{positive semifield}, a condition that is met, e.g., by the semirings of Booleans and non-negative reals.

Such characterisation allows us to give a handy definition of the \emph{canonical weak lifting} of $\P$ over $\EM{\S}$ (Theorem~\ref{thm:weaklift}) and to observe that such lifting is \emph{almost} the same as the monad $\mathcal{C} \colon \EM{\S} \to \EM{\S} $ defined by Jacobs in~\cite{jacobs_coalgebraic_2008} (Remark~\ref{remarkJacobs}): the only difference is the absence in $\mathcal{C}$ of the empty subset. Such difference becomes crucial when considering the composed monads, named $\mathcal{C}\mathcal M \colon \Set \to \Set$ in~\cite{jacobs_coalgebraic_2008} and $\convpowS\colon \Set \to \Set$ in this paper: the latter maps a set $X$ into the set of convex subsets of $\S X$, while the former additionally requires the subsets to be non-empty. It turns out that while $\Kl{\mathcal{C}\mathcal M}$ is not $\mathbb{CPPO}$-enriched, a necessary condition for the coalgebraic framework in \cite{hasuo_generic_2006}, $\Kl{\convpowS}$ indeed is (Theorem \ref{thm:Kleisli is CPPO enriched}). 

Composing monads by means of weak distributive laws is rewarding in many respects: here we exploit the fact that algebras for the composed monad $\convpowS$ coincide with $\delta$-algebras, namely algebras for both $\P$ and $\S$ satisfying a certain pentagonal law. One can extract from this law some distributivity axioms that, together with the axioms for semimodules (algebras for the monad $\S$) and those for complete semilattices (algebras for the monad $\P$), provide an algebraic theory presenting the monad $\convpowS$ (Theorem \ref{thm:pres}).

We conclude by coming back to the finite powerset monad $\P_f$. By replacing, in the above theory, complete semilattices with semilattices with bottom (algebras for the monad $\P_f$) one obtains a theory  presenting the monad $\convpowfS$ of \emph{finitely generated} convex subsets (Theorem \ref{thm:convpowfS is presented by (Sigma', E')}), which is formally defined as a restriction of the canonical $\convpowS$. The theory, displayed in Table \ref{tab:axiomsinitial}, consists of the theory presenting the monad $\P_f$ and the theory presenting the monad $\S$ with four distributivity axioms.

\begin{table}[t]
    \centering
    \caption{The sets of axioms $E_{\SL}$ for semilattices (left), $E_{\LSM}$ for $S$-semimodules (right) and $E_{\D'}$ for their distributivity (bottom).}
\begin{tabular}{|c|}
    \begin{tabular}{rcl||rclrcl}
    \hline
$(x \sqcup y) \sqcup z $&=&$ x\sqcup(y \sqcup z)$ \; &\; $(x+y)+z$ &=& $x+(y+z)$ \;&\; $ (\lambda +_S \mu) \cdot x$ &=& $\lambda \cdot x + \mu \cdot x $  \\
     $x\sqcup y $&=&$ y \sqcup x$     \;&\; $x+y $&=&$ y+x$     \;&\; $0_S \cdot x $&=&$ 0$  \\
    $x\sqcup \bot $&=&$ x$ \; &\; $x+0 $&=&$ x$ \; &\;  $(\lambda \mu) \cdot x $&=&$ \lambda \cdot (\mu \cdot x)$   \\
    $x\sqcup x $&=&$ x$ \;& && &\;  $\lambda \cdot (x+y)$&=& $\lambda \cdot x + \lambda \cdot y$ \\
    && \;& && &\; $\lambda \cdot 0$&=& $0$ \\
    \hline
    \hline
    \end{tabular}\\
   \begin{tabular}{rclrcl}
       %\hline
       $\lambda \cdot \bot$ & =& $\bot \text{ for } \lambda\neq 0_S$ \quad & \quad $\lambda \cdot (x \sqcup y) $&=&$(\lambda \cdot x) \sqcup (\lambda \cdot y)$ \\
        $x + \bot $&=&$\bot$ \quad & \quad $x + (y \sqcup z) $&=&$ (x + y) \sqcup (x + z)$ \\
   \end{tabular}\\
       \hline
   \end{tabular}
    \label{tab:axiomsinitial}
\end{table}

\begin{notation}
			 We assume the reader to be familiar with monads and their maps.	Given a monad $(M,\eta^M,\mu^M)$ on $\C$, $\EM M$ and $\Kl M$ denote, respectively, the Eilenberg-Moore category and the Kleisli category of $M$. The latter is defined as the category whose objects are the same as $\C$ and a morphism $f \colon X \to Y$ in $\Kl M$ is a morphism $f \colon X \to M(Y)$ in $\C$. We write $\U M \colon \EM M \to \C$ and $\UK M \colon \Kl M \to \C$ for the canonical forgetful functors, and $\F M \colon \C \to \EM M$, $\FK M \colon \C \to \Kl M$ for their respective left adjoints. Recall, in particular, that $\F M (X) = (X,\mu^M_X)$ and, for $f \colon X \to Y$, $\F M(f) = M(f)$. Given $n$ a natural number, we denote by $\natset n$ the set $\{ 1, \dots, n\}$.
\end{notation}

\section{(Weak) Distributive laws}

Given two monads $S$ and $T$ on a category $\C$, is there a way to compose them to form a new monad $ST$ on $\C$? This question was answered by Beck~\cite{beck_distributive_1969} and his theory of \emph{distributive laws}, which are natural transformations $\delta \colon TS \to ST$ satisfying four axioms and that provide a canonical way to endow the composite functor $ST$ with a monad structure. We begin by recalling the classic definition. In the following, let $(T,\eta^T,\mu^T)$ and $(S,\eta^S,\mu^S)$ be two monads on a category $\C$.

\begin{definition}\label{def:distributive law}
	A \emph{distributive law} of the monad $S$ over the monad $T$ is a natural transformation $\delta\colon TS \to ST$ such that the following diagrams commute.
	\begin{equation}\label{eqn:def distributive law}
	\begin{tikzcd}[ampersand replacement=\&,column sep=2em]
	TSS \ar[r,"\delta S"] \ar[d,"T\mu^S"'] \& STS \ar[r,"S\delta"] \& SST \ar[d,"\mu^S T"] \& \& TTS \ar[r,"T\delta"] \ar[d,"\mu^T S"'] \& TST \ar[r,"\delta T"] \& STT \ar[d,"S\mu^T"] \\
	TS \ar[rr,"\delta"] \& \& ST \& \& 	TS \ar[rr,"\delta"] \& \& ST \\
	\& T \ar[dl,"T\eta^S"'] \ar[dr,"\eta^S T"] \& \& \& \& S \ar[dl,"\eta^T S"'] \ar[dr,"S \eta^T"] \\
	TS \ar[rr,"\delta"] \& \& ST \& \& 	TS \ar[rr,"\delta"] \& \& ST
	\end{tikzcd}	
	\end{equation}
\end{definition}

One important result of Beck's theory is the bijective correspondence between distributive laws, liftings to Eilenberg-Moore algebras and extensions to Kleisli categories, in the following sense.

\begin{definition}
%	\begin{itemize}
		A \emph{lifting} of the monad $S$ to $\EM T$ is a monad $(\tilde S, \eta^{\tilde S}, \mu^{\tilde S} )$ where
		\[
		\begin{tikzcd}
		\EM T \ar[r,"\tilde S"] & \EM T \\
		\C \ar[r,"S"] \ar[u,"\F T"] & \C \ar[u,"\F T"']
		\end{tikzcd} \quad \text{commutes,} \quad \U T \eta^{\tilde S} = \eta^S \U T, \quad \U T \mu^{\tilde S} = \mu^S \U T.
		\]
		An \emph{extension} of the monad $T$ to $\Kl S$ is a monad $(\tilde T, \eta^{\tilde T}, \mu^{\tilde T})$ such that
		\[
		\begin{tikzcd}
		\C \ar[r,"T"] \ar[d,"\FK S"'] & \C \ar[d,"\FK S"] \\
		\Kl S \ar[r,"\tilde T"] & \Kl S
		\end{tikzcd} \quad \text{commutes,} \quad \eta^{\tilde T} \FK S = \FK S \eta^T, \quad \mu^{\tilde T} \FK S = \FK S \mu^T.
		\]
%	\end{itemize}
\end{definition}

Böhm~\cite{bohm_weak_2010} and Street~\cite{street_weak_2009} have studied various weaker notions of distributive law; here we shall use the one that consists in dropping the axiom involving $\eta^T$ in Definition~\ref{def:distributive law}, following the approach of Garner~\cite{garner_vietoris_2020}.

\begin{definition}\label{def:weak distributive law}
	 A \emph{weak distributive law of $S$ over $T$} is a natural transformation $\delta \colon TS \to ST$ such that the diagrams in~\eqref{eqn:def distributive law} regarding $\mu^S$, $\mu^T$ and $\eta^S$ commute.
\end{definition}

There are suitable weaker notions of liftings and extensions which also bijectively correspond to weak distributive laws as proved in~\cite{bohm_weak_2010,garner_vietoris_2020}.

\begin{definition}\label{def:weak lifting}
	A \emph{weak lifting} of $S$ to $\EM T$ consists of a monad $(\tilde S,\eta^{\tilde S},\mu^{\tilde S})$ on $\EM{T}$ and two natural transformations
	\[
	\begin{tikzcd}
	\U T \tilde S \ar[r,"\iota"] & S \U T \ar[r,"\pi"] & \U T \tilde S
	\end{tikzcd}
	\]
	such that $\pi \iota = id_{\U T \tilde S}$ and such that the following diagrams commute:
	\begin{equation}\label{eqn:weak lifting diagrams iota}
	\begin{tikzcd}
	\U T \tilde S \tilde S \ar[r,"\iota \tilde S"] \ar[d,"\U T \mu^{\tilde S}"'] & S \U T \tilde S \ar[r,"S \iota"] & S S \U T \ar[d,"\mu^S \U T"] \\
	\U T \tilde S \ar[rr,"\iota"] & & S \U T
	\end{tikzcd}
	\quad
	\begin{tikzcd}
	& \U T \ar[dl,"\U T \eta^{\tilde S}"'] \ar[dr,"\eta^S \U T"] \\
	\U T \tilde S \ar[rr,"\iota"] & & S \U T
	\end{tikzcd}
	\end{equation}
	\begin{equation}\label{eqn:weak lifting diagrams pi}
	\begin{tikzcd}
	S S \U T \ar[r,"S\pi"] \ar[d,"\mu^S \U T"'] & S \U T \tilde S \ar[r,"\pi \tilde S"] & \U T \tilde S \tilde S \ar[d,"\U T \mu^{\tilde S}"] \\
	S \U T \ar[rr,"\pi"] & & \U T \tilde S
	\end{tikzcd}
	\quad
	\begin{tikzcd}
	& \U T \ar[dl,"\eta^S \U T"'] \ar[dr,"\U T \eta^{\tilde S}"] \\
	S \U T \ar[rr,"\pi"] & & \U T \tilde S
	\end{tikzcd}
	\end{equation}
	A \emph{weak extension} of $T$ to $\Kl{S}$ is a functor $\tilde T \colon \Kl S \to \Kl S$ together with a natural transformation $\mu^{\tilde T} \colon \tilde T \tilde T \to \tilde T$ such that $\FK S T = \tilde T \FK S$ and $\mu^{\tilde T} \FK S = \FK S \mu^T$.
\end{definition}

\begin{theorem}[\cite{beck_distributive_1969,bohm_weak_2010,garner_vietoris_2020}]\label{thm:bijective correspondence (weak) distributive laws and (weak) rest}
	There is a bijective correspondence between (weak) distributive laws $TS \to ST$, (weak) liftings of $S$ to $\EM T$ and (weak) extensions of $T$ to $\Kl S$.
\end{theorem}

\input{monads}

\input{distributiveLaw}

\input{weakLifting}

\input{composedMonad}

\section{Conclusions: Related and Future Work}
Our work was inspired by~\cite{goy_combining_2020} where Goy and Petrisan compose the monads of powerset  and probability distributions by means of a weak distributive law in the sense of Garner~\cite{garner_vietoris_2020}. Our results also heavily rely on the work of Clementino et al.~\cite{clementino_monads_2014} that illustrates necessary and sufficient conditions on a semiring $S$ for the existence of a weak distributive law $\delta \colon \S \P \to \P \S$. However, to the best of our knowledge, the alternative characterisation of $\delta$ provided by Theorem~\ref{thm:delta for positive refinable semifields} was never shown.

Such characterisation is essential for giving a handy description of the lifting $\tilde\P\colon \EM{\S} \to \EM{\S}$ (Theorem~\ref{thm:weaklift}) as well as to observe the strong relationships with the work of Jacobs (Remark~\ref{remarkJacobs}) and the one of Klin and Rot (Remark~\ref{remarkKlinRot}). The weak distributive law $\delta$ also plays a key role in providing the algebraic theories presenting the composed monad $\convpowS$ (Theorem~\ref{thm:weaklift}) and its finitary restriction $\convpowfS$ (Theorem \ref{thm:convpowfS is presented by (Sigma', E')}). These two theories resemble those appearing in, respectively,~\cite{goy_combining_2020} and~\cite{bonchi2019theory} where the monad of probability distributions plays the role of the monad $\S$ in our work.

Theorem~\ref{thm:Kleisli is CPPO enriched} allows to reuse the framework of coalgebraic trace semantics~\cite{hasuo_generic_2006} for modelling over $\Kl{\convpowS}$ systems with both nondeterminism and quantitative features. The alternative framework based on coalgebras over $\EM{\convpowS}$ directly leads to \emph{nondeterministic weighted automata}. A proper comparison with those in~\cite{droste2009handbook} is left as future work. Thanks to the abstract results in~\cite{DBLP:conf/csl/BonchiPPR14}, language equivalence for such coalgebras could be checked by means of coinductive up-to techniques. It is worth remarking that, since $\delta$ is a weak distributive law, then thanks to the work in~\cite{DBLP:journals/corr/abs-2010-00811}, up-to techniques are also sound for ``convex-bisimilarity'' (in coalgebraic terms, behavioural equivalence for the lifted functor $\tilde\P \colon \EM{\S} \to \EM{\S}$).

We conclude by recalling that we have two main examples of positive semifields: $\Bool$ and $\mathbb{R}^+$. Booleans could lead to a coalgebraic modal logic and trace semantics for \emph{alternating automata} in the style of \cite{DBLP:conf/fossacs/KlinR15}.
For $\mathbb{R}^+$, we hope that exploiting the ideas in \cite{DBLP:journals/tcs/Rutten05} our monad could shed some lights on the behaviour of linear dynamical systems featuring some sort of nondeterminism.

%% ---- Bibliography ----
%%
%% BibTeX users should specify bibliography style 'splncs04'.
%% References will then be sorted and formatted in the correct style.
%%
\bibliographystyle{splncs04}
\bibliography{FoSSaCS2021biblio}

\section{Appendix to Section~\ref{sec:the weak distributive law}}\label{Appendix Section 4} 

\subsubsection{Proof of Theorem~\ref{thm:existence of delta}}\label{proof:def of delta}. We calculate $\delta$ by first analysing the canonical extension $\widetilde{\mon S}$ of $\mon S$ to $\Rel$ and then by following the proof of Theorem~\ref{thm:bijective correspondence (weak) distributive laws and (weak) rest}.

The formula to extend the functor $\mon S$ to $\Rel$ is the following: on objects it behaves like $\mon S$; if $R \subseteq A \times B$ is a relation and $\pi_A \colon R \to A$ and $\pi_B \colon R \to B$ are the two projections, we have that
\[
\widetilde{\mon S} (R)= \Graph{\mon S(\pi_B)} \circ \Graph{\mon S(\pi_A)}^{-1}
\]
(composition performed in $\Rel$) where $\Gamma$ computes the graph of a function. 
Now, for all $\psi \colon R \to S$
	\[
	\mon S (\pi_A) (\psi) = \Bigl( a \mapsto \sum_{r \in \pi_A^{-1}\{a\}} \psi(r) \Bigr) = 
	\Bigl( a \mapsto \sum_{\substack{b \in B \\ a R b}} \psi(a,b) \Bigr)
	\]
	and similarly for $\mon S (\pi_B)(\psi)$. So, 
	\[
	\widetilde{\mon S}(R) = \Bigl\{ \Bigl( \bigl( a \mapsto \sum_{\substack{b \in B \\ a R b}} \psi(a,b)  \bigr) ,   \bigl( b \mapsto \sum_{\substack{a \in A \\ a R b}} \psi(a,b)  \bigr)   \Bigr) \mid \psi \colon R \to S,\, \supp\psi \text{ finite} \Bigr\}.
	\]
From this we obtain a weak distributive law as follows. Recall the identity-on-objects isomorphism of categories $F \colon \Kl\P \to \Rel$ where for $f \colon X \to \pow(Y)$ in $\Set$, $F(f)=\{(x,y) \mid y \in f(x)\} \subseteq X \times Y$ and for $R \subseteq X \times Y$, $\pre F (R) = ( x \mapsto \{y \in Y \mid x \mathrel R y\}) \colon X \to \pow(Y)$. Consider $\id{\pow X} \colon \pow X \to \pow X$ in $\Set$. This is a Kleisli map $\id{\pow X} \colon \pow X \to X$ in $\Kl \P$, which corresponds to the relation ${}\ni_X{} \colon \pow X \to X = \{ (A,x) \mid x \in A  \}$. Then we have
	\[
	\widetilde{\mon S}(\ni_X)=\Bigl\{ \Bigl( \bigl( A \mapsto \sum_{x \in A } \psi(A,x)  \bigr) ,   \bigl( x \mapsto \sum_{A \ni x} \psi(A,x)  \bigr)   \Bigr) \mid \psi \in \mon S (\ni_X) \Bigr\}.
	\]
	Remember that $\widetilde{\mon S}$ and $\mon S$ coincide on objects. This relation, seen back as a Kleisli map $\widetilde{\mon S}\pow X \to \widetilde{\mon S} X$, gives us the $X$-th component of the desired weak distributive law, which is indeed~(\ref{eqn:def of delta}). \qed

\subsubsection{The Proof of Theorem~\ref{thm:delta for positive refinable semifields}.}\label{proof:delta for positive refinable semifields} We need a lemma that generalises the distributivity property in an arbitrary semiring. %, showing how the observation of equation~(\ref{eqn:example generalised distributivity}) can always be done. 
The statement of the following Lemma makes sense only for commutative semirings, but it can be adapted for arbitrary semirings by restricting it to sets $L$ only of the form $\natset n = \{1,\dots,n\}$.

\begin{lemma}\label{lemma:generalised distributivity}
	%Let $\S=(S,+,\cdot,0,1)$ be a commutative semiring. 
	For all $n \in \N$, $L \subset \N$ such that $\card L = n$, for all $(s_k)_{k \in L} \in \N^n$, for all $(\lambda^k_j)_{\substack{k \in L \\ j \in \natset{s_k}}}$ family of elements of $S$:
	\begin{equation}\label{eqn:generalised distributivity}
	\sum_{w \in \prod\limits_{k \in L} \natset {s_k}} \, \prod_{k \in L} \lambda^k_{w_k} = 
	\prod_{k \in L} \, \sum_{j=1}^{s_k} \lambda^k_j.
	\end{equation}
\end{lemma}
\begin{proof}
	By induction on $n$. If $n=0$, both sides of~(\ref{eqn:generalised distributivity}) are $0$. 
	
	%	If $n=1$, for $L={\ast}$ we have that~\ref{eqn:generalised distributivity} reduces to
	%	\[
	%	\sum_{j \in \natset{s_\ast}} \lambda_j = \sum_{j=1}^{s_\ast} \lambda_j.
	%	\]
	Let now $n \ge 0 $, suppose that the statement of the Lemma holds for $n$, let us consider a $L \subset \N$ such that $\card L = n+1$, $(s_k)_{k \in L} \in \N^{n+1}$, $(\lambda^k_j)_{\substack{k \in L \\ j \in \natset{s_k}}}$. Without loss of generality we can assume that $L=\natset{n+1}$. Then we have:
	\begin{align*}
	\sum_{w \in \prod\limits_{k =1}^{n+1} \natset{s_k}} \, \prod_{k =1}^{n+1} \lambda^k_{w_k} &=
	\sum_{j=1}^{s_{n+1}} \biggl( \sum_{w' \in \prod\limits_{k =1}^{n} \natset{s_k}} \, \Bigl((\prod_{k =1}^{n} \lambda^k_{w_k'}) \cdot \lambda^{n+1}_j \Bigr) \biggr) \\
	&= \sum_{j=1}^{s_{n+1}} \biggl( \Bigl( \sum_{w' \in \prod\limits_{k =1}^{n} \natset{s_k}} \, \prod_{k =1}^{n} \lambda^k_{w_k'}\Bigr) \cdot \lambda^{n+1}_j  \biggr) \\
	&= \sum_{j=1}^{s_{n+1}} \biggl( \Bigl( \prod_{k=1}^n \, \sum_{u=1}^{s_k} \lambda^k_u \Bigr) \cdot \lambda^{n+1}_j \biggr) \\
	&= \Bigl( \prod_{k=1}^n \, \sum_{u=1}^{s_k} \lambda^k_u \Bigr) \cdot \Bigl( \sum_{j=1}^{s_{n+1}} \lambda^{n+1}_j  \Bigr) \\
	&= \prod_{k=1}^{n+1} \, \sum_{j=1}^{s_k} \lambda^k_j. 
	\end{align*} \qed
\end{proof}

\begin{theorem}\label{thm:delta=delta'}
	If $\S$ is a positive, refinable semifield, then 
	for all $X$ set and $\Phi \in \mon S \pow X$, $\delta_X(\Phi) = \convclos{\choice \Phi}{\muS_X}$.
\end{theorem}
\begin{proof}
	We shall discuss first some preliminary cases involving the empty set, excluding each time all the cases previously covered. Recall that 
	\[
	\delta_X(\Phi)=\Biggl\{ \phi \in \mon S(X) \mid \exists \psi \in \mon S (\ni) \ldotp \begin{cases}
	\forall A \in \pow X \ldotp \Phi(A) = \sum_{x \in A} \psi(A,x) \\
	\forall x \in X \ldotp \phi(x) = \sum_{A \ni x} \psi(A,x)
	\end{cases}  
	\Biggr\}
	\]
	and write, within the scope of this proof, $\delta_X'(\Phi) = \convclos{\choice \Phi}{\muS_X}$, that is:
	\[
	\delta_X'(\Phi)= \left\{ \muS_X(\Psi) \mid \Psi \in \mon S^2 X\ldotp \sum\limits_{\chi \in \mon S X} \Psi(\chi)=1, \, \supp \Psi \subseteq \choice\Phi \right\}
	\]
	where
	\[
	\choice \Phi =\{ \chi \in \mon S X \mid \exists u \in \prod\limits_{A \in \supp \Phi} A \ldotp \forall x \in X \ldotp \chi(x) = \sum\limits_{\substack{A \in \supp \Phi \\ x=u_A}} \Phi(A)  \}
	\]
	(this is of course an equivalent formulation of~\eqref{eqn:c(Phi)}). 
	\paragraph{Case 1: $\Phi(\emptyset)\ne 0$.} We have that $\delta_X(\Phi)=\emptyset$ because there is no $\psi \in \mon S (\ni)$ such that $\Phi(\emptyset)=\sum_{x \in \emptyset} \psi(A,x)=0$. At the same time, also $\delta_X'(\Phi)=\emptyset$, because $\choice{\Phi}=\emptyset$--since $\prod_{A \in \supp \Phi} A = \emptyset$--hence there is no $\Psi \in \mon S ^2 X$ with empty support that can satisfy $\sum_{\chi \in \mon S X} \Psi(\chi)=1$.
	
	\paragraph{Case 2: $X = \emptyset$.} (We also assume that $\Phi(\emptyset)=0$ from now on.) We have that $\mon S \pow \emptyset = \{ \Omega \colon \{\emptyset\} \to S \}$, therefore $\Phi = 0 \colon \{\emptyset\} \to S$. Moreover, ${}\ni{} \subseteq \pow(\emptyset) \times \emptyset = \emptyset$ and $\mon S (\emptyset) = \{\emptyset \colon \emptyset \to S\}$ is the singleton of the empty map, so
	\begin{align*}
	\delta_\emptyset (\Phi) &= \{ \phi \in \mon S (\emptyset) \mid \exists \psi \in \mon S (\emptyset) \ldotp \forall A \subseteq \emptyset \ldotp \Phi(A)=\sum_{x \in \emptyset} \psi(A,x) \} \\
	&= \{\phi \in \mon S (\emptyset) \mid \Phi(\emptyset)=0\} \\
	&= \{ \emptyset \colon \emptyset \to S  \}
	\end{align*}
	because, by assumption, $\Phi(\emptyset) = 0$. On the other hand, we have that, since $\supp \Phi = \emptyset$, 
	\[
	\choice\Phi = \{ \chi \in \mon S (\emptyset) \mid \exists u \in \prod\limits_{A \in \emptyset} A  \} = \mon S (\emptyset)
	\]
	because the zero-ary product is a choice of a terminal object of $\Set$, a singleton. So, 
	\begin{align*}
	\delta_\emptyset'(\Phi) &= \{ \mu(\Psi) \mid \Psi \in \mon S ^2 \emptyset, \sum_{\chi \in \mon S \emptyset} \Psi(\chi)=1\} \\
	&= \{ \mu(\mon S (\emptyset) \ni \emptyset \mapsto 1)  \} = \{\emptyset \colon \emptyset \to S  \}.
	\end{align*}
	
	\paragraph{Case 3: $\Phi = 0 \colon \pow X \to S$.} (We also assume that $X \ne \emptyset$ from now on.) We have that the only $\psi \in \mon S (\ni)$ such that for all $\sum_{x \in A} \psi(A,x) = 0$ for all $A \subseteq X$ is the null function, therefore $\delta_X(0 \colon \pow X \to S)=\{0 \colon X \to S\}$. On the other hand, we have that $\supp \Phi = \emptyset$, so
	\[
	\choice\Phi = \{ \chi \in \mon S X \mid \exists u \in \prod_{A \in \emptyset} A \ldotp \forall x \in X \ldotp \chi(x)=\sum_{A \in \emptyset} \Phi(A) \} = \{0 \colon X \to S\}.
	\]
	It follows then that
	\begin{align*}
	\delta_X'(0 \colon \pow S \to S) &= \{\mu(\Psi) \mid \Psi \in \mon S ^2 X, \sum_{\chi \in \mon S X} \Psi(\chi)=1, \supp \Psi \subseteq \{0 \colon X \to S\} \} \\
	&= \{ \mu(\mon S X \ni 0 \mapsto 1)  \} \\
	&= \{ (S \ni x \mapsto 1 \cdot 0)  \} = \{0 \colon X \to S\}.
	\end{align*}
	
	We have now discussed all the preliminary cases. For the rest of the proof, we shall assume that $X \ne \emptyset$, $\Phi \ne 0 \colon \pow X \to S$ and $\Phi(\emptyset) = 0$.
	
	We first prove $\delta_X(\Phi) \subseteq \delta_X'(\Phi)$. To this end, let $\phi \in \mon S X$ and $\psi \in \mon S (\ni)$ such that $\Phi(A) = \sum_{x \in A} \psi(A,x)$ for all $A \in \pow X$ and $\phi(x)=\sum_{A \ni x} \psi(A,x)$ for all $x \in X$. Observe that:
	\begin{itemize}
		\item for all $(A,x) \in \supp \psi$ we have that $A \in \supp \Phi$ and $x \in \supp \phi \cap A$,
		\item for all $A \in \supp \Phi$ there is $x \in \supp \phi \cap A$ such that $(A,x) \in \supp \psi$,
		\item for all $x \in \supp \phi$ there exists $A \in \supp \Phi$ such that $A \ni x$ and $(A,x) \in \supp \psi$.
	\end{itemize}
	Hence the first elements of pairs in $\supp \psi$ range over all and only elements of $\supp \Phi$, and the second entries of pairs in $\supp \psi$ range over all and only elements of $\supp \phi$. In other words, say $\supp \Phi = \{A_1,\dots,A_n\}$: then we have
	\[
	\supp \psi = \Bigl\{ (A_1,x^1_1), \dots, (A_1,x^1_{s_1}), \dots, (A_n, x^n_1), \dots, (A_n,x^n_{s_n}) \Bigr\}
	\]
	where $\bigcup_{i=1}^n \{x^i_1,\dots,x^i_{s_i}\} = \supp \phi$. Notice that for all $i \in \natset{n}$, $u,v \in \natset{s_i}$, if $u \ne v$ then $x^i_u \ne x^i_v$, but we may have $x^i_a = x^j_b$ if $i \ne j$, because the $A_i$'s are distinct but not \emph{disjoint}. We can then write:
	\begin{enumerate}[(I)]
		\item\label{enumerate 1 lemma delta = delta'}  $\Phi(A_i) = \sum\limits_{j=1}^{s_i} \psi(A_i,x^i_j)$ for all $i \in \natset{n}$
		\item\label{enumerate 2 lemma delta = delta'}  $\phi(x) = \sum\limits_{\substack{i \in \natset n \\ \exists j \in \natset{s_i}\ldotp x = x^i_j}} \psi(A_i,x)$ for all $x \in X$.
	\end{enumerate}
	We want to find a convex linear combination $\Psi$ of elements of $\mon S X$ of the form $\sum_{i=1}^n \Phi(A_i) a_i$ for some $a_i \in A_i$ such that $\mu(\Psi)=\phi$. Now, for every $i \in \natset{n}$, we have many candidates for $a_i$, namely $x^i_1, \dots, x^i_{s_i}$. Given that every $\chi \in \supp \Psi$ can only involve one $x^i_j$ for every $i$, we shall have as many $\chi$'s as the number of ways to pick one element of $\{ x^i_1,\dots,x^i_{s_i} \}$ for each $i$: let then $w \in \prod_{i=1}^n \natset{s_i}$ (it is a tuple of indexes $w_i$ which we are going to use as $j$'s). Define $\chi_w$ as
	\[
	\chi_w(x^1_{w_1}) = \Phi(A_1), \,\dots,\, 
	\chi_w(x^n_{w_n}) = \Phi(A_n) 
	\]
	with the understanding that if $x^i_{w_i} =  x^j_{w_j}$ then $\chi_w(x^i_{w_i}) = \Phi(A_i) + \Phi(A_j)$. Written more precisely, we define for all $x \in X$
	\[
	\chi_w (x) = \sum_{\substack{i \in \natset n \\ x = x^i_{w_i}}} \Phi(A_i).
	\]
	We now define $\Psi \in \mon S ^2 X$ with $\supp \Psi=\{ \chi_w \mid w \in \prod_{i=1}^n \natset{s_i}  \}$ that assigns to each $\chi_w$ a number such that $\sum_w \Psi(\chi_w) = 1$. We shall use the previous Lemma to show that by defining
	\[
	\Psi(\chi_w) = \frac{ \prod_{i=1}^n \psi(A_i,x^i_{w_i}) } { \prod_{i=1}^n \Phi(A_i) }
	\]
	for all $w$, or more precisely by defining
	\[
	\Psi(\chi) = \sum_{\substack{w \in \prod_{i=1}^n \natset{s_i} \\ \chi=\chi_w}} \frac{ \prod_{i=1}^n \psi(A_i,x^i_{w_i}) } { \prod_{i=1}^n \Phi(A_i) }
	\]
	for all $\chi \in \mon S X$ we indeed have that $\Psi$ satisfies the conditions of $\delta_X'(\Phi)$ and $\phi=\mu(\Psi)$. First we prove that $\Psi$ is in fact a convex linear combination:
	\begin{align*}
	\sum_{\chi \in \supp \Psi} \Psi(\chi) &= \frac 1 {\prod_{i=1}^n \Phi(A_i)} \sum_{w \in \prod \natset{s_i}} \prod_{i=1}^n \psi(A_i,x^i_{w_i}) \\
	&= \frac 1 {\prod_{i=1}^n \Phi(A_i)} \prod_{i=1}^n \sum_{j=1}^{s_i} \psi(A_i,x^i_j) & \text{Lemma~\ref{lemma:generalised distributivity}} \\
	&= \frac 1 {\prod_{i=1}^n \Phi(A_i)} {\prod_{i=1}^n \Phi(A_i)}  & \text{because of \eqref{enumerate 1 lemma delta = delta'}} \\
	&=1.
	\end{align*}
	Next, notice that for every $w$ the vector $u \in \prod_{i=1}^n A_i$ required by the definition of $\delta_X'(\Phi)$ is exactly $(x^1_{w_1},\dots,x^n_{w_n})$. Finally, we compute $\mu(\Psi)(x)$ for an arbitrary $x \in X$. The equations marked with $(\ast)$ will be explained later.
	\begin{align*}
	\muS_X(\Psi)(x) &= \sum_{\chi \in \supp \Psi} \Psi(\chi) \cdot \chi(x) \\
	&= \sum_{w \in \prod_{k=1}^n \natset{s_k}} \Biggl[ \frac {\prod_{i=1}^n \psi(A_i,x^i_{w_i})} {\prod_{i=1}^n \Phi(A_i)} \sum_{\substack{i \in \natset{n} \\ x=x^i_{w_i}}} \Phi(A_i) \Biggr] & (\ast_1) \\
	&= \sum_{w \in \prod_{k=1}^n \natset{s_k}} \, \sum_{\substack{i \in \natset{n} \\ x=x^i_{w_i}}} \Biggl[   \frac {\prod_{k=1}^n \psi(A_k,x^k_{w_k})} {\prod_{k=1}^n \Phi(A_k)} \Phi(A_i)  \Biggr] \\
	&= \sum_{i \in \natset{n}} \,\, \sum_{\substack{w \in \prod_{k=1}^n \natset{s_k} \\ x=x^i_{w_i} }} \Biggl[   \frac {\prod_{k=1}^n \psi(A_k,x^k_{w_k})} {\prod_{k=1}^n \Phi(A_k)} \Phi(A_i)  \Biggr] \\
	&= \sum_{\substack{i \in \natset n \\ \exists j \in \natset{s_i} \ldotp x = x^i_j}} \, \sum_{\substack{w \in \prod_{k=1}^n \natset{s_k} \\ x=x^i_{w_i} }} \Biggl[ \Phi(A_i) \frac {\psi(A_i,x^i_{w_i})} {\prod_{k=1}^n \Phi(A_k)} \prod_{\substack{k \in \natset n \\ k \ne i}} \psi(A_k,x^k_{w_k}) \Biggr] & (\ast_2) \\
	&= \sum_{\substack{i \in \natset n \\ \exists j \in \natset{s_i} \ldotp x = x^i_j}} \Phi(A_i)  \cdot \frac {\psi(A_i,x)} {\prod_{k=1}^n \Phi(A_k)} \cdot \sum_{\substack{w \in \prod_{k=1}^n \natset{s_k} \\ x=x^i_{w_i} }} \, \prod_{\substack{k \in \natset n \\ k \ne i}} \psi(A_k,x^k_{w_k})
	\end{align*}
	Now, use Lemma~\ref{lemma:generalised distributivity} with $L=\{1,\dots,i-1,i+1,\dots,n\}$ and $\lambda^k_j = \psi(A_k,x^k_j)$ in the following chain of equations:
	\begin{align*}
	\sum_{\substack{w \in \prod_{k=1}^n \natset{s_k} \\ x=x^i_{w_i} }} \, \prod_{\substack{k \in \natset n \\ k \ne i}} \psi(A_k,x^k_{w_k}) &= \sum_{w' \in \prod_{k \in L} \natset{s_k} } \, \prod_{k \in L} \psi(A_k,x^k_{w_k'}) \\
	&= \prod_{k \in L} \, \sum_{j=1}^{s_k} \psi(A_k, x^k_j)  \\
	&= \prod_{\substack{k \in \natset n \\ k \ne i}} \, \sum_{j=1}^{s_k} \psi(A_k,x^k_j) \\
	&= \prod_{\substack{k \in \natset n \\ k \ne i}} \Phi(A_k)
	\end{align*}
	Therefore, we obtain
	\begin{align*}
	\muS_X(\Phi) &= \sum_{\substack{i \in \natset n \\ \exists j \in \natset{s_i} \ldotp x = x^i_j}} \Phi(A_i)  \cdot \frac {\psi(A_i,x)} {\prod_{k=1}^n \Phi(A_k)} \cdot \prod_{\substack{k \in \natset n \\ k \ne i}} \Phi(A_k) \\
	&= \sum_{\substack{i \in \natset n \\ \exists j \in \natset{s_i} \ldotp x = x^i_j}} \psi(A_i,x) \\
	&= \phi(x) & \text{because of \eqref{enumerate 2 lemma delta = delta'}}.
	\end{align*}
	It remains to explain equations $\ast_1$ and $\ast_2$ above. The latter is simply due to the fact that if $i$ is such that there is no $j$ for which $x =x^i_j$, then there is not a $w$ such that $x=x^i_{w_i}$ either, and vice versa. The former is more delicate. It may be the case that $\chi_w = \chi_{w'}$ for $w \ne w'$. So, let us write
	\[
	\supp \Psi = \{\chi_{w^1}, \dots, \chi_{w^m}\}
	\]
	where the $\chi_{w^l}$ are now all distinct. Then we have
	\begin{align*}
	\sum_{\chi \in \supp \Psi} \Psi(\chi) \cdot \chi(x) &= \sum_{l=1}^m \Psi(\chi_{w^l}) \cdot \chi_{w^l}(x) \\
	&= \sum_{l=1}^m \Biggl( \sum_{\substack{w \\ \chi_{w^l} = \chi_w}} \frac{ \prod_{i=1}^n \psi(A_i,x^i_{w_i}) } { \prod_{i=1}^n \Phi(A_i) } \Biggr) \chi_{w^l}(x) \\
	&= \sum_{l=1}^m\, \sum_{\substack{w \\ \chi_{w^l} = \chi_w}} \frac{ \prod_{i=1}^n \psi(A_i,x^i_{w_i}) } { \prod_{i=1}^n \Phi(A_i) } \chi_{w^l}(x) \\
	&= \sum_{l=1}^m\, \sum_{\substack{w \\ \chi_{w^l} = \chi_w}} \frac{ \prod_{i=1}^n \psi(A_i,x^i_{w_i}) } { \prod_{i=1}^n \Phi(A_i) } \chi_w(x) \\
	&= \sum_{w} \frac{ \prod_{i=1}^n \psi(A_i,x^i_{w_i}) } { \prod_{i=1}^n \Phi(A_i) } \chi_w(x)
	\end{align*}
	which is the right-hand side of $(\ast_1)$.
	
	The other inclusion, $\delta_X'(\Phi) \subseteq \delta_X(\Phi)$, is easier. Let $\Psi \in \mon S ^2 X$ be such that $\sum_{\chi \in \mon S X} \Psi(\chi) = 1$ and 
	\[
	\supp \Psi \subseteq \{ \chi \in \mon S X \mid \exists u \in \prod\limits_{A \in \supp \Phi} A \ldotp \forall x \in X \ldotp \chi(x) = \sum\limits_{\substack{A \in \supp \Phi \\ x=u_A}} \Phi(A)  \}.
	\]
	Write again $\supp \Phi=\{A_1,\dots,A_n\}$ and let $\supp \Psi = \{\chi_1,\dots,\chi_m\}$. For all $j \in \natset m$, let $u^j \in \prod_{i=1}^n A_i$ be such that $\chi_j(x) = \sum_{\substack{i \in \natset n \\ x = u^j_i}} \Phi(A_i)$. Then we have
	\[
	\muS_X(\Psi)(x) = \sum_{j=1}^m \Psi(\chi_j) \chi_j(x) = \sum_{j=1}^m \Psi(\chi_j) \sum_{\substack{i \in \natset n \\ x = u^j_i}} \Phi(A_i).
	\]
	Define then
	\[
	\psi(B,x) = \begin{cases}
	\sum_{\substack{j \in \natset m \\ x=u^j_i }} \Psi(\chi_j) \Phi(A_i) & \exists (!) i \in \natset n \ldotp B=A_i \\
	0 & \text{otherwise}
	\end{cases}
	\]
	%	\[
	%	\psi(A_i,x) = \sum_{\substack{j \in \natset m \\ x=u^j_i }} \Psi(g_j) \Phi(A_i), \psi(B,x) = 0
	%	\]
	We have $\supp \psi = \{ (A_i,u^j_i) \mid i \in \natset n, j \in \natset m  \}$ is finite and so $\psi \in \mon S (\ni)$. We now verify the two conditions required in the definition of $\delta_X(\Phi)$. For all $x \in X$:
	\begin{align*}
	\sum_{A \ni x} \psi(A,x) &= \sum_{i=1}^n \psi(A_i,x) = \sum_{i=1}^n \sum_{\substack{j \in \natset m \\ x=u^j_i }} \Psi(\chi_j) \Phi(A_i) = \sum_{j=1}^m \sum_{ \substack{i \in \natset n \\ x = u^j_i} } \Psi(\chi_j) \Phi(A_i) \\
	&= \muS_X(\Psi)(x)
	\end{align*}
	while for all $A \subseteq X$, if $A \notin \supp \Phi$, then $\sum_{x \in A} (\psi(A,x)) = 0 = \Phi(A)$ by definition and, for all $i \in \natset n$:
	\begin{align*}
	\sum_{x \in A_i} \Psi(A_i,x) &= \sum_{x \in A_i} \sum_{\substack{j \in \natset m \\ x = u^j_i}} \psi(\chi_j) \Phi(A_i) \overset{\ast_3}{=} \sum_{j=1}^m \Psi(\chi_j) \Phi(A_i) = \Phi(A_i) \sum_{j=1}^m \Psi(\chi_j) \\
	&= \Phi(A_i)
	\end{align*}
	where equation $\ast_3$ holds because for all $j \in \natset m$ the addend $\Psi(\chi_j)\Phi(A_i)$ appears on the left-hand side exactly once, when in the first sum we are using, as $x$, precisely $u^j_i \in A_i$. In other words, given $j$, there is a unique $x \in A_i$ such that $x=u^j_i$, so $\Psi(\chi_j)\Phi(A_i)$ appears in the left-hand side of $\ast_3$ and it does so only once. % Vice versa, for all $x \in A_i$ and for all $j$ such that $x=u^j_i$, 
	Therefore $\muS_X(\Psi) \in \delta_X(\Phi)$, and the proof is complete. \qed
\end{proof}

%%%%%%%%%%%%%%%%%%%%%%%%%%%%%%%%%%%%%%%%%%%%%%%%%%%%%%%%%%%%%%%%%%%%%%%%%%%%%%%%%%%%%%%%%%%%%%%%%%%%%%%

\section{Appendix to Section~\ref{sec: the weak lifting}}\label{Appendix Section 5}

\subsubsection{Proof of Lemma~\ref{lemma:delta of Dirac}.}
First of all, if $A = \emptyset$, then the left-hand side is empty, because there is no $\psi \in \mon S (\ni)$ such that $1=\Delta_{\emptyset}(\emptyset) = \sum_{x \in \emptyset} \psi(B,x) = 0$, and so is the right-hand side, as the only function whose support is empty is the null function, which cannot satisfy $\sum_{x \in X} \phi(x) = 1$. 
	
	Suppose then that $A \neq \emptyset$. 
	For the left-to-right inclusion: observe first of all that $\supp \phi \ne \emptyset$, because otherwise we would have that $\psi(B,x)=0$ for all $x \in X$ and for all $B \ni x$ due to the fact that $\S$ is a positive semiring. This would then lead to the contradiction $1=\Delta_A(A) = \sum_{x \in A} \psi(A,x) = 0$. Let then $x \in \supp \phi$: then $\phi(x) \ne 0$, hence there exists $B \ni x$ such that $\psi(B,x) \ne 0$. It is necessarily the case, however, that $B=A$, because if $B \ne A$, then $0 = \Delta_A(B) = \sum_{y \in B} \psi(B,y) \ge \psi(B,x) \ne 0$, which is a contradiction. Thus $\supp \phi \subseteq A$. Moreover,
	\[
	\sum_{x \in X} \phi(x) = \sum_{x \in A} \phi(x) = \sum_{x \in A} \sum_{B \ni x} \psi(B,x) \overset{\ast}{=} \sum_{x \in A} \psi(A,x) = \Delta_A(A)=1
	\]
	where equation $\ast$ holds because for all $B \ne A$ and for all $y \in B$ we have $\psi(B,y)=0$, hence the only addend of $\sum_{B \ni x} \psi(B,x)$ which is possibly not-null is the ``$A$-th one'': $\psi(A,x)$.
	
	Vice versa, let $\phi \colon X \to S$ such that $\supp \phi$ is finite and contained in $A$ and satisfying $\sum_{x \in X} \phi(x)=1$. Then the map
	\[
	\begin{tikzcd}[ampersand replacement=\&,row sep=0em]
	\ni \ar[r,"\psi"] \& \Rp \\
	(B,x) \ar[r,|->] \& \begin{cases}
	\phi(x) & B=A \\
	0 & \text{otherwise}
	\end{cases}
	\end{tikzcd}
	\] 
	has finite support, as it is in bijective correspondence with $\supp \phi$; for all $B \in \pow X$, if $B \ne A$ then $\sum_{x \in B} \psi(B,x)=0$, otherwise 
	\[
	\sum_{x \in B} \psi(B,x) = \sum_{x \in A} \phi(x) = \sum_{x \in X} \phi(x)=1;
	\]
	and finally, for all $x \in X$, 
	$
	\sum\limits_{B \ni x} \psi(B,x) = \psi(A,x)=\phi(x). 
	$ \qed

\subsubsection{The proof of Proposition~\ref{cor:image along a of convex subset of S(X) is convex in X}.}\label{proof:image along a convex subset of S(X) is convex in X} We actually state and prove a more precise result, that implies Proposition~\ref{cor:image along a of convex subset of S(X) is convex in X}.

\begin{proposition}\label{prop:image along a of convex closures}
	Let $(X,a)$ be a $\mon S$-algebra and let $\mathcal A \subseteq \mon S X$. Then
	\[
	\pow a \Bigl( \convclos{\mathcal A}{\muS_X} \Bigr) = \convclos{\pow a (\mathcal A)}{a}.
	\]
\end{proposition}
\begin{proof}
	We have
	\[
	\pow a \Bigl( \convclos{\mathcal A}{\muS_X} \Bigr) = \{ a(\muS_X(\Psi)) \mid \Psi \in \mon S^2 X, \sum_{\psi \in \mon S X} \Psi(\psi) = 1, \supp \Psi \subseteq \mathcal A \}
	\]
	while
	\[
	\convclos{\pow a (\mathcal A)}{a} = \{ a(\phi) \mid \phi \in \mon S X, \sum_{x \in X} \phi(x)=1, \supp \phi \subseteq \{a(\psi) \mid \psi \in \mathcal A\}   \}.
	\]
	For the left-to-right inclusion: let $\Psi \in \mon S^2 X$ be such that $\sum_{\psi \in \mon S X} \Psi(\psi) = 1$ and with $\supp \Psi \subseteq \mathcal A$. Define $\phi=\mon S (a) (\Psi)$. Then $x \in \supp \phi$ if and only if there exists $\psi \in \supp \Psi$ such that $x = a(\psi)$. Hence, if $x \in \supp \phi$, then $x=a(\psi)$ for some $\psi \in \mathcal A$. Moreover,
	\[
	\sum_{x \in X} \phi(x) = \sum_{x \in X} \sum_{\substack{\psi \in \supp \Psi \\ a(\psi)=x}} \Psi(\psi) = \sum_{\psi \in \mon S X} \Psi(\psi) = 1.
	\]
	Since $a(\muS_X(\Psi)) = a(\mon S (a)(\Psi))$ by the properties of $(X,a)$ as $\mon S$-algebra, and $\mon S (a) (\Psi) = \phi$ by definition, we conclude.
	
	Vice versa, given $\phi \in \mon S X$ such that $\sum_{x \in X} \phi(x)=1$ and with $\supp \phi = \{a(\psi_1),\dots,a(\psi_n)\}$ where $\psi_1,\dots,\psi_n \in \mathcal A$, define
	\[
	\Psi = \left( 
	\begin{tikzcd}[row sep=0em]
	\psi_1 \ar[r,|->] & \phi(a(g_1)) \\
	\vdots & \vdots \\
	\psi_n \ar[r,|->] & \psi(a(f_n))
	\end{tikzcd}
	\right)
	\]
	Then $\sum_{\psi \in \mon S X} \Psi(\psi) = 1$ and
	\[
	a(\muS_X(\Psi)) = a(\mon S (a) (\Psi)) = a(\phi)
	\]
	because $\mon S (a) (\Psi) = \phi$. \qed
\end{proof}

\subsubsection{Details of the Proof of Proposition~\ref{thm:convpow(X,a) is a S-algebra}.}\label{proof:details proof convpow(X,a) is a S-algebra} In the same notations set up in the proof of Theorem~\ref{thm:convpow(X,a) is a S-algebra}, let $A \in \supp \Phi$. Define for all $x \in X$:
	\[
	\chi_A(x) = \sum_{\substack{\phi \in \supp \Psi \\ x=u^\phi(A)}} \Psi(\phi).
	\]
	Then $\supp{\chi_A} = \{u^\phi(A) \mid \phi\in \supp \Psi\}$ is finite, hence $\chi_A \in \mon S X$. Let now 
	\[
	u(A) = a(\chi_A).
	\]
	We have that
	\[
	\sum_{x \in X} \chi_A(x) = \sum_{x \in X} \sum_{\substack{\phi \in \supp \Psi \\ x=u^\phi(A)}} \Psi(\phi) = \sum_{\phi \in \supp \Psi} \Psi(\phi)=1,
	\]
	hence $u(A) \in A$ because $A$ is convex. Define, for all $x \in X$:
	\[
	\psi(x) = \sum_{\substack{A \in \supp \phi \\ x=a(\chi_A)}} \Phi(A).
	\]
	Then $\supp \psi = \{ u(A) \mid A \in \supp \Phi\}$ so $\psi \in \mon S X$.
	To conclude, we have to prove that $a(\muS_X(\Psi)) = a(\psi)$. To that end, let for all $\chi \in \mon S X$
	\[
	\Psi'(\chi) = \sum_{\substack{ A \in \supp \phi \\ \chi_A = \chi}} \Phi(A).
	\]
	Then $\supp \Psi' = \{ \chi_A \mid A \in \supp \Phi\}$ is finite. Then we have for all $x \in X$ that
	\begin{align*}
	\muS_X(\Psi')(x) &= \sum_{\chi \in \supp \Psi'} \Psi'(\chi) \cdot \chi(x) \\
	&\overset{\ast}{=} \sum_{A \in \supp \Phi} \Phi(A) \cdot \sum_{\substack{\phi \in \supp \Psi \\ x=u^\phi(A)}} \Psi(\phi) \\
	&= \sum_{\phi \in \supp \Psi} \Psi(\phi) \cdot \sum_{\substack{A \in \supp \Phi \\ x=u^\phi(A)}} \Phi(A) \\
	&= \sum_{\phi \in \mon S X} \Psi(\phi) \cdot \phi(x) \\
	&= \muS_X(\Psi)(x)
	\end{align*}
	where equation $\ast$ is explained in a similar way than $(\ast_1)$ in the proof of Theorem~\ref{thm:delta=delta'}. We also have that 
	\[
	\mon S (a)(\Psi')(x) = \sum_{\chi \in a^{-1}\{x\}} \Psi'(\chi) = \sum_{\substack{\chi \in \mon S X \\ a(\chi) = x}} \sum_{\substack{A \in \supp \Phi \\ \chi_A = \chi}} \Phi(A) = \sum_{\substack{A \in \supp \Phi \\ a(\chi_A)=x}} \Phi(A) = \psi(x).
	\]
	Therefore,
	\[
	a(\psi)=a\bigl(\mon S (a)(\Psi')  \bigr) = a\bigl(\muS_X (\Psi')\bigr) = a\bigl(\muS_X(\Psi)\bigr).
	\]
	
	%The other inclusion is much easier. Let $\psi \in \mon S X$ and $u \in \prod \supp \Phi$ such that $\psi(x) = \sum_{\substack{A \in \supp \Phi \\ x=u(A)}} \Phi(A)$. Consider simply $\Psi = \etaS_{\mon S X} (\psi) = \Delta_g \in \mon S ^2 X$. Then $\Psi$ obviously satisfies all the required properties and $\muS_X(\Psi)=\psi$. 

%%%%%%%%%%%%%%%%%%%%%%%%%%%%%%%%%%%%%%%%%%%%%%%%%%%%%%%%%%%%%%%%%%%%%%%%%%%%%%%%%%%%%%%%%%%%%%%%%%

\subsubsection{Continuation of the Proof of Theorem~\ref{thm:weaklift}.}\label{proof:weak lifting of P to EM(S)} We continue the proof of Theorem~\ref{thm:weaklift} from p.~\pageref{end of section 5}, by analising the action of $\tilde \P$ on morphisms and the unit and the multiplication of $\tilde \P$.

First of all, technically for $f \colon (X,a) \to (X',a')$ in $\EM \S$, $\tilde\P(f)$ is defined as 
\[
\tilde\P(f)(\mathcal A) = \convclos{\P(f)(\mathcal A)}{a'} \quad \text{for all $A \in \ConvPow X a$,}
\]
however it is not difficult to see that $\P(f)(\mathcal A)$ is convex in $(X',a')$ using the fact that $f$ is a morphism of $\S$-algebras. 
%We follow the recipe given by the proof of Theorem~\ref{thm:bijective correspondence (weak) distributive laws and (weak) rest}. We have already computed the idempotent~(\ref{eqn:idempotent e}) as the function that given $A \in \pow X$ returns $\convclos A a$, so $\pi$ and $\iota$ are, as already noticed, the convex closure operator and the set-inclusion of $\ConvPow X a$ into $\pow X$ respectively. This provides us $\tilde\P (f)$ as described.
	
	The unit of the monad $\tilde \P$ is given, for every $(X,a)$ object of $\EM{\mon S}$, as the unique morphism $\eta^{\tilde \P}_{(X,a)} \colon (X,a) \to \tilde \P (X,a)$ that makes the two triangles of~(\ref{eqn:weak lifting diagrams iota}) and~(\ref{eqn:weak lifting diagrams pi}) commute, which are in our case:
	\[
	\begin{tikzcd}
	X \ar[r,"\eta^{\P}_X"] \ar[d,"\U{\mon S} (\eta^{\tilde \P}_{(X,a)})"'] & \pow X \\
	\ConvPow X a \ar[ur,"\iota_{(X,a)}"',hookrightarrow]
	\end{tikzcd}
	\qquad
	\begin{tikzcd}[column sep=4em]
	X \ar[r,"\U{\mon S}(\eta^{\tilde \P}_{(X,a)})"] \ar[d,"\eta^{\P}_X"'] & \ConvPow X a \\
	\pow X \ar[ur,"\convclos{(-)}{a}"']
	\end{tikzcd}
	\]
	By definition of $\iota$ and $\eta^{\P}$, the only arrow that makes the left triangle above commutative is necessarily   
	\[
	\begin{tikzcd}[row sep=0em]
	(X,a) \ar[r,"\eta^{\tilde \P}_{(X,a)}"] & (\ConvPow X a,\alpha_a) \\
	x \ar[r,|->] & \{x\}
	\end{tikzcd}
	\]
	and this makes also the right triangle commutative, since $\convclos{\{x\}} a = \{ x \}$.
	Similarly, the multiplication $\mu^{\tilde \P}$ is defined, for every $\mon S$-algebra $(X,a)$, as the unique morphism making the two rectangles of~(\ref{eqn:weak lifting diagrams iota}) and~(\ref{eqn:weak lifting diagrams pi}) commute. These are in our case:
	\[
	\begin{tikzcd}[column sep=1.2em]
	\ConvPow {\bigl(\ConvPow X a\bigr)} {\alpha_a} \ar[r,hookrightarrow,"\iota"] \ar[d,"\mu^{\tilde \P}_{(X,a)}"'] & \pow(\ConvPow X a) \ar[r,"\pow \iota",hookrightarrow] & \pow \pow X \ar[d,"\mu^\P_X"] \\
	\ConvPow X a \ar[rr,hookrightarrow,"\iota"] & & \pow X
	\end{tikzcd}
	\quad 
	\begin{tikzcd}[column sep=1.5em]
	\pow \pow X \ar[r,"\pow \convclos{(-)} a"] \ar[d,"\mu^{\P}_X"'] 
	& \pow(\ConvPow X a) \ar[r,"\convclos{(-)}{\alpha_a}"] & \ConvPow {(\ConvPow X a)} {\alpha_a}  \ar[d,"{\mu^{\tilde \P}_{(X,a)}}"] \\
	\pow X \ar[rr,"\convclos{(-)} a"] & & \ConvPow X a
	\end{tikzcd}
	\]
	
	By definition of $\iota$ and $\mu^{\P}$, the only arrow that makes the left rectangle above commutative is necessarily  
	\[
	\begin{tikzcd}[row sep=0em]
	\ConvPow {(\ConvPow X a)} {\alpha_a} \ar[r,"\mu^{\tilde \P}_{(X,a)}"] & \ConvPow X a \\
	\mathcal A \ar[r,|->] & \bigcup_{A \in \mathcal A} A
	\end{tikzcd}
	\]
	for all $(X,a) \in \EM{\mon S}$.
	The reader may wonder why such morphism is well defined, namely why $\mu^{\tilde \P}_{(X,a)}(\mathcal A)$ is in $\ConvPow X a$. This is the case because the abstract results guarantee existence and uniqueness of a such morphism. The interested reader may find next a more concrete proof of this and also the fact that this $\mu^{\tilde \P}$ makes the right rectangle commutative.
	
	\begin{proposition}\label{prop:union of convex family of convex subsets is convex}
	Let $(X,a)$ be a $\mon S$-algebra and $\mathcal A \in \ConvPow {\bigl(\ConvPow X a\bigr)} {\alpha_a}$. Then $\bigcup_{A \in \mathcal A} A$ is convex in $(X,a)$.
\end{proposition}
\begin{proof}
	We want to prove the following equality:
	\[
	\bigcup_{A \in \mathcal A} A = \{a(\phi) \mid \phi \in \mon S X, \sum_{x \in X} \phi(x)=1, \supp \phi \subseteq \bigcup_{A \in \mathcal A} A\}.
	\]
	The left-to-right inclusion is trivial: for all $A \in \mathcal A$ and $x \in A$, $x=a(\Delta_x)$. Vice versa, let $\phi \in \mon S X$ be such that $\sum_{x \in X} \phi(x)=1$ and $\supp \phi \subseteq \bigcup_{A \in \mathcal A} A$. We want to find an element $A \in \mathcal A$ such that $a(\phi) \in A$. Suppose
	\[
	\supp \phi = \{x_1,\dots,x_n\}.
	\]
	Then for all $i \in \natset n$ there exists $A_i \in \mathcal A$ such that $x_i \in A_i$. Notice that if $i \ne j$ then it is not necessarily the case that $A_i \ne A_j$. Define $\Phi \in \mon S (\ConvPow X a)$ as
	\[
	\Phi(B) = \sum_{ \substack{i \in \natset n \\ B=A_i} } \phi(x_i)
	\]
	for all $B \in \ConvPow X a$.
	Then $\supp \Phi = \{A_i \mid i \in \natset n\}$ is finite and
	\[
	\sum_{B \in \ConvPow X a} \Phi(B) = \sum_{B \in \ConvPow X a} \sum_{ \substack{i \in \natset n \\ B=A_i} } \phi(x_i) = \sum_{i \in \natset n} \phi(x_i) = 1.
	\]
	Since $\mathcal A$ is convex in $(\ConvPow X a,\alpha_a)$, we have that $\alpha_a(\Phi) \in \mathcal A$. This is going to be our desired $A$: we shall prove that $a(\phi) \in \alpha_a(\Phi)$.
	
	We have $\alpha_a(\Phi)=\pow a (\delta_X(\mon S (i)(\Phi)))$, hence, if we prove that $\phi \in \delta_X(\mon S (i)(\Phi))$, we have finished. To this end, recall the characterisation of $\delta_X$ using elements of $\mon S (\ni)$, for ${}\ni {} \subseteq \pow X \times X$:
	\begin{multline*}
	\delta_X(\mon S (i) (\Phi)) = \\
	\left\{a(\chi) \mid \chi \in \mon S X \ldotp \exists \psi \in \mon S (\ni) \ldotp 
	\begin{cases}
	\forall B \subseteq X \ldotp \mon S (i) (\Phi)(B) = \sum_{x \in B} \psi(B,x) \\
	\forall x \in X \ldotp \chi(x) = \sum_{B \ni x} \psi(B,x)
	\end{cases}
	\right\}.
	\end{multline*}
	Define, for all $(B,x) \in \pow X \times X$ such that $B \ni x$:
	\[
	\psi(B,x) = \sum_{ \substack{i \in \natset n \\ (B,x)=(A_i,x_i)}} \phi(x_i).
	\]
	Then $\supp \psi = \{(A_i,x_i) \mid i \in \natset n\}$ is finite, for all $B \subseteq X$
	\[
	\sum_{x \in B} \psi(B,x) = \sum_{x \in B} \sum_{ \substack{i \in \natset n \\ (B,x)=(A_i,x_i)}} \phi(x_i) = \sum_{\substack{i \in \natset n \\ B=A_i}} \phi(x_i) = \mon S (i)(\Phi)(B)
	\]
	and, for all $x \in X$,
	\[
	\sum_{B \ni x} \psi(B,x) = \sum_{B \ni x} \sum_{ \substack{i \in \natset n \\ (B,x)=(A_i,x_i)}} \phi(x_i) = 
	\begin{cases}
	0 & \forall i \in \natset n \ldotp x \ne x_i \\
	\phi(x_i) & \exists (!) i \in \natset n \ldotp x=x_i
	\end{cases}
	= \phi(x)
	\]
	where if there is $i$ such that $x=x_i$, then such $i$ is unique due to the fact that the $x_j$'s are distinct. This proves that $\phi \in \delta_X(\mon S (i)(\Phi))$. \qed
\end{proof}
%
%
% COMMENTED THE OLD TEXT
%
%
%We can have therefore the full right of defining a function
%\[
%\begin{tikzcd}[row sep=0em]
%\ConvPow {(\ConvPow X a)} {\alpha_a} \ar[r,"\mu^{\tilde \P}_{(X,a)}"] & \ConvPow X a \\
%\mathcal A \ar[r,|->] & \bigcup_{A \in \mathcal A} A
%\end{tikzcd}
%\]
%for all $(X,a) \in \EM{\mon S}$: this function makes the rectangle in~(\ref{eqn:weak lifting diagrams iota}) commute. Let us now look at the rectangle in~(\ref{eqn:weak lifting diagrams pi}), which in our case is:
%\begin{equation}\label{eqn:commutative rectangle for mu Ptilde}
%\begin{tikzcd}
%\pow \pow X \ar[r,"\pow \convclos{(-)} a"] \ar[d,"\mu^{\P}_X"'] 
%& \pow(\ConvPow X a) \ar[r,"\convclos{(-)}{\alpha_a}"] & \ConvPow {(\ConvPow X a)} {\alpha_a}  \ar[d,"{\mu^{\tilde \P}_{(X,a)}}"] \\
%\pow X \ar[rr,"\convclos{(-)} a"] & & \ConvPow X a
%\end{tikzcd}
%\end{equation}
%This diagram commutes if and only if for all $\mathcal U \subseteq \pow X$, 
%\[
%\convclos{\bigcup_{U \in \mathcal U} U} a = \mu^{\tilde \P}_{(X,a)} \bigl( \{ \alpha_a(\phi) \mid \phi \in \mon S \ConvPow X a, \sum_{C \in \ConvPow X a} \phi(C)=1, \supp \phi \subseteq \{ \convclos U a \mid U \in \mathcal U \ \} \bigr)
%\]

%where we denoted with $\overline B$ the convex closure in $(X,a)$ of a subset $B \subseteq X$. 
\begin{proposition}
	The following diagram commutes.
	\[
	\begin{tikzcd}
	\pow \pow X \ar[r,"\pow \convclos{(-)} a"] \ar[d,"\mu^{\P}_X"'] 
	& \pow(\ConvPow X a) \ar[r,"\convclos{(-)}{\alpha_a}"] & \ConvPow {(\ConvPow X a)} {\alpha_a}  \ar[d,"{\mu^{\tilde \P}_{(X,a)}}"] \\
	\pow X \ar[rr,"\convclos{(-)} a"] & & \ConvPow X a
	\end{tikzcd}
	\]	
\end{proposition}
\begin{proof}
	In this proof we shall simply write $\overline A$ for $\convclos A a$ for any $A \subseteq X$, because there is no risk of confusion.
	We have to show that
	\[
	\{ a(\phi) \mid \phi \in \mon S X, \sum_{x \in X} \phi(x)=1, \supp \phi \subseteq \bigcup_{U \in \mathcal U} U \}
	=
	\bigcup_{\substack{\Phi \in \mon S \ConvPow X a \\
			\sum \Phi(C)=1 \\
			\supp \Phi \subseteq \{ \convclos U {} \mid U \in \mathcal U \}
	}}
	\alpha_a(\Phi)   .
	\]
	For the left-to-right inclusion: let $\phi \in \mon S X$ such that $\sum_{x \in X} \phi(x)=1$, and suppose $\supp \phi = \{x_1,\dots,x_n\}$. We have that for all $i \in \natset n$ there exists $A_i \in \mathcal U$ such that $x_i \in A_i$. While the $x_i$'s are distinct, the $A_i$'s need not be. Define, for all $B \in \ConvPow X a$,
	\[
	\Phi(B)=\sum_{ \substack{i \in \natset n \\ B=\overline{A_i} }} \phi(x_i).
	\]
	Then $\supp \Phi = \{ \overline {A_1},\dots,\overline{A_n} \}$ hence $\Phi \in \mon S \ConvPow X a$. One can then prove that $a(\phi) \in \alpha_a(\Phi)$ in the same way as showed in the proof of Proposition~\ref{prop:union of convex family of convex subsets is convex}.
	
	Vice versa, an element of the right-hand side is of the form $a(\psi)$ for some $\Phi \in \mon S \ConvPow X a$ such that $\sum_{B \in \ConvPow X a} \Phi(B) = 1$, $\supp \Phi = \{\overline {A_1},\dots,\overline{A_n}\}$ with $A_i \in \mathcal U$ for all $i \in \natset n$ and for some $\psi \in \mon S X$ and $u \in \prod_{i =1}^n \overline{A_i}$ such that 
	\[
	\psi(x)=\sum_{ \substack{i \in \natset n \\[.25em] x=u_i} } \Phi(\overline{A_i}).
	\]
	We want to prove that $a(\psi)= a(\phi)$ for an appropriate $\phi \in \mon S X$ such that $\sum_{x \in X} \phi(x) =1$ and $\supp \phi \subseteq \bigcup_{U \in \mathcal U} U$.
	
	Since $u_i \in \overline{A_i}$, we have that $u_i = a(\phi_i)$ for some $\phi_i \in \mon S X$ such that $\sum_{x \in X} \phi_i(x)=1$ and $\supp \phi_i \subseteq A_i$. These $\phi_i$'s are not necessarily distinct. Let then, for all $\phi \in \mon S X$,
	\[
	\Psi(\phi) = \sum_{ \substack{i \in \natset n \\ \phi=\phi_i} } \Phi(\overline{A_i}).
	\]
	Then $\supp \Psi = \{f_1,\dots,f_n\}$ so $\Psi \in \mon S ^2 X$ and it is easy to prove, by means of direct calculations, that $\psi=\mon S (a)(\Phi)$, $\sum_{x \in X} \muS_X(\Psi)=1$ and that $\supp{\muS_X(\Psi)} \subseteq \bigcup_{U \in \mathcal U} U$. Then we have that $a(\psi)=a(\mon S (a)(\Psi)) = a(\muS_X(\Psi))$, and $\muS_X(\Psi)$ is the $\phi$ we were looking for. \qed
\end{proof}

\input{appendixSectionSix}

\vfill

{\small\medskip\noindent{\bf Open Access} This chapter is licensed under the terms of the Creative Commons\break Attribution 4.0 International License (\url{http://creativecommons.org/licenses/by/4.0/}), which permits use, sharing, adaptation, distribution and reproduction in any medium or format, as long as you give appropriate credit to the original author(s) and the source, provide a link to the Creative Commons license and indicate if changes were made.}

{\small \spaceskip .28em plus .1em minus .1em The images or other third party material in this chapter are included in the chapter's Creative Commons license, unless indicated otherwise in a credit line to the material.~If material is not included in the chapter's Creative Commons license and your intended\break use is not permitted by statutory regulation or exceeds the permitted use, you will need to obtain permission directly from the copyright holder.}

\medskip\noindent\includegraphics{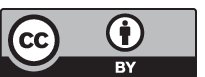}

\end{document}

%% file: monads.tex
\section{The Powerset and Semimodule Monads}\label{sec:monad}

\subsubsection{The Monad $\P$.} Let us now consider, as $S$, the \emph{powerset} monad $(\pow,\eta^\P,\mu^\P)$, where $\etaP_X(x)=\{x\}$ and $\muP_X(\mathcal U) = \bigcup_{U \in \mathcal U} U$. Its algebras are precisely the complete semilattices and we have that $\Kl \P$ is isomorphic to the category $\Rel$ of sets and relations. Hence, giving a distributive law $T\P \to \P T$ is the same as giving an extension of $T$ to $\Rel$: for this to happen the notion of weak cartesian functor and natural transformation is crucial.
\begin{definition}
	A functor $T \colon \Set \to \Set$ is said to be \emph{weakly cartesian} if and only if it preserves weak pullbacks. A natural transformation $\phi \colon F \to G$ is said to be \emph{weakly cartesian} if and only if its naturality squares are weak pullbacks.
\end{definition}
Kurz and Velebil~\cite{kurz_relation_2016} proved, using an original argument of Barr~\cite{barr_relational_1970}, that an endofunctor $T$ on $\Set$ has at most one extension to $\Rel$ and this happens precisely when it is weakly cartesian; similarly a natural transformation $\phi \colon F \to G$, with $F$ and $G$ weakly cartesian, has at most one extension $\tilde \phi \colon \tilde F \to \tilde G$, precisely when it is weakly cartesian. The following result is therefore immediate.
\begin{proposition}[{\cite[Corollary 16]{garner_vietoris_2020}}]\label{prop:(weak)distributive law with P iff weakly cartesian}
	For any monad $(T,\eta^T,\mu^T)$ on $\Set$:
	\begin{enumerate}
		\item There exists a unique distributive law of $\P$ over $T$ if and only if $T$, $\eta^T$ and $\mu^T$ are weakly Cartesian.
		\item There exists a unique \emph{weak} distributive law of $\P$ over $T$ if and only if $T$ and $\mu^T$ are weakly Cartesian.
	\end{enumerate}
\end{proposition}

\subsubsection{The Monad $\mon S$.} Recall that a \emph{semiring} is a tuple $(S,+,\cdot,0,1)$ such that $(S,+,0)$ is a commutative monoid, $(S,\cdot,1)$ is a monoid, $\cdot$ distributes over $+$ and $0$ is an annihilating element for $\cdot$. In other words, a semiring is a ring where not every element has an additive inverse. Natural numbers $\N$ with the usual operations of addition and multiplication form a semiring. Similarly, integers, rationals and reals form semirings. Also the booleans $\Bool=\{0,1\}$ with $\vee$ and $\wedge$ acting as $+$ and $\cdot$, respectively, form a semiring. 

Every semiring $S$ generates a \emph{semimodule} monad $\S$ on $\Set$ as follows. Given a set $X$, $\S(X) = \{ \phi \colon X \to S \mid \supp \phi \text{ finite} \}$, where $\supp \phi = \{ x \in X \mid \phi(x) \ne 0 \}$. For $f \colon X \to Y$, define for all $\phi\in \S (X)$
\[
\S(f)(\phi) = \Bigl( y \mapsto \sum_{x \in f^{-1}\{y\}} \phi(x) \Bigr) \colon Y \to S.
\]
This makes $\S$ a functor. The unit $\etaS_X \colon X \to \S(X)$ is given by $\etaS_X(x)=\Delta_x$, where $\Delta_x$ is the Dirac function centred in $x$, while the multiplication $\muS_X \colon \S^2(X) \to \S(X)$ is defined for all $\Psi\in \S^2(X)$ as
\[
\muS_X(\Psi) = \Bigl( x \mapsto \sum_{\phi \in \supp \Psi} \Psi(\phi) \cdot \phi(x)  \Bigr) \colon X \to S.
\]
An algebra for $\S$ is precisely a \emph{left-S-semimodule}, namely a set $X$ equipped with a binary operation $+$, an element $0$ and a unary operation $\lambda \cdot$ for each $\lambda\in S$, satisfying the equations in Table~\ref{tab:axiomsinitial}.
Indeed, if $X$ carries a semimodule structure then one can define a map $a \colon \S X \to X$ as, for $\phi \in \mon S X$,
\begin{equation}\label{eq:S-algebra associated to semimodule}
a(\phi) = \sum_{x \in X} \phi(x) \cdot x
\end{equation}
where the above sum is finite because so is $\supp \phi$. Vice versa, if $(X,a)$ is an $\mon S$-algebra, then the corresponding left-semimodule structure on $X$ is obtained by defining for all $\lambda \in S$ and $x,y \in X$
\begin{equation}\label{eqn:semimodule associated to S-algebra}    x+^a y = a (x \mapsto 1, y \mapsto 1), \qquad
	0^a = a(\zero), \qquad
	\lambda \cdot^a x = a (x\mapsto \lambda).
\end{equation}
Above and in the remainder of the paper, we write the list $(x_1\mapsto s_1, \dots, x_n \mapsto s_n)$ for the only function $\phi\colon X \to S$ with support $\{x_1,\dots , x_n\}$ mapping $x_i$ to $s_i$ and we write the empty list $\zero$ for the function constant to $0$. For instance, for $a=\muS_X \colon \mon S \mon S X \to \mon S X$, the left-semimodule structure is defined for all $\phi_1,\phi_2 \in \mon S X$ and $x\in X$ as
\[(\phi_1+^{\mu^{\mon S}}\phi_2)(x)=\phi_1(x)+\phi_2(x), \qquad 0^{\mu^{\mon S}}(x)=0, \qquad (\lambda \cdot^{\mu^{\mon S}} \phi_1 )(x)= \lambda \cdot \phi_1( x).\]

Proposition~\ref{prop:(weak)distributive law with P iff weakly cartesian} tells us exactly when a (weak) distributive law of the form $T\P \to \P T$ exists for an arbitrary monad $T$ on $\Set$. Take then $T=\S$: when are the functor $\S$ and the natural transformations $\etaS$ and $\muS$ weakly cartesian? The answer has been given in~\cite{clementino_monads_2014} (see also~\cite{gumm_monoid-labeled_2001}), where a complete characterisation in purely algebraic properties for $S$ is provided. In Table~\ref{tab:semiring properties} we recall such properties.

\begin{table}[t]
    \centering
     \caption{Definition of some properties of a semiring $S$. Here $a,b,c,d \in S$.}
    \begin{tabular}{|c|p{30em}|}
    \hline
    Positive     & $a+b=0 \implies a=0=b$ \\ \hline
    Semifield    & $a \ne 0 \implies \exists x \ldotp a \cdot x = x \cdot a = 1$ \\ \hline
    Refinable    & $a+b=c+d \implies \exists x,y,z,t \ldotp x+y=a,\, z+t=b, \, x+z=c, \,y+t = d$ \\ \hline
    (A) & $a+b=1 \implies a = 0$ or $b=0$ \\ \hline
    (B) & $a \cdot b = 0 \implies a=0 \text{ or } b=0$ \\ \hline
    (C) & $a+c=b+c \implies a = b$ \\ \hline
    (D) & $\forall a,b$. $\exists x \ldotp a+x = b$ or $b+x = a$\\ \hline
    (E) & $a+b = c \cdot d \implies \exists t \colon \{ (x,y) \in S^2 \mid x+y=d  \} \to S$ such that \newline
    $
    \sum\limits_{x+y=d} t(x,y) x = a, \quad \sum\limits_{x+y=d} t(x,y)y=b, \quad \sum\limits_{x+y=d} t(x,y)=c.
    $ \\ \hline
    \end{tabular}
    \label{tab:semiring properties}
\end{table}

\begin{theorem}[{\cite{clementino_monads_2014}}]\label{thm:S,etaS,muS weakly cartesian iff}
	Let $S$ be a semiring.
	\begin{enumerate}
		\item The functor $\S$ is weakly cartesian if and only if $S$ is positive and refinable.
		\item $\etaS$ is weakly cartesian if and only if $S$ enjoys \emph{(A)} in Table \ref{tab:semiring properties}.
		\item\label{condition:muS weakly cartesian iff} If $\S$ is weakly cartesian, then $\muS$ is weakly cartesian if and only if $S$ enjoys \emph{(B)} and \emph{(E)} in Table \ref{tab:semiring properties}.
	\end{enumerate}
\end{theorem}

\begin{remark}\label{rem:S positive refinable semifield then ok}
In~\cite[Proposition 9.1]{clementino_monads_2014} it is proved that if $S$ enjoys (C) and (D), then $S$ is refinable; if $S$ is a positive semifield, then it enjoys (B) and (E). In the next Proposition we prove that if $S$ is a positive semifield then it is also refinable, hence $\S$ and $\mu^\S$ are weakly cartesian.
\end{remark}

\begin{proposition}\label{prop:positive semifield implies refinable}
If $S$ is a positive semifield, then it is refinable.
\end{proposition}
\begin{proof}
Let $a$, $b$, $c$ and $d$ in $S$ be such that $a+b=c+d$. If $a+b=0$, then take $x=y=z=t=0$, otherwise take
\[
x=\frac{ac}{c+d}, \quad y = \frac{ad}{c+d}, \quad z=\frac{bc}{c+d}, \quad t=\frac{bd}{c+d}.
\]
Then $x+y=a,\, z+t=b, \, x+z=c, \,y+t = d$.\qed
\end{proof}

%It follows that if $S$ is a positive semifield, then $\S$ and $\mu^\S$ are weakly cartesian.

\begin{example}\label{example:distributive law for N}
	It is known that, for $\S=\N$,  a distributive law $\delta \colon \mon S \P \to \P \mon S$ exists. Indeed one can check  that all conditions of Theorem~\ref{thm:S,etaS,muS weakly cartesian iff} are satisfied, therefore we can apply Proposition~\ref{prop:(weak)distributive law with P iff weakly cartesian}.1. In this case, the monad $\mon S X$ is naturally isomorphic to the commutative monoid monad, which given a set $X$ returns 
	the collection of all \emph{multisets} of elements of $X$. 
	The law $\delta$ is well known (see e.g.~\cite{garner_vietoris_2020,DBLP:journals/entcs/HylandNPR06}): given a multiset $\langle A_1,\dots,A_n \rangle$ of subsets of $X$ in $\mon S \pow X$, where the $A_i$'s need not be distinct, it returns the set of multisets
		$\{ \langle a_1, \dots, a_n \rangle \mid a_i\in A_i \}$.
\end{example}

\subsubsection{Convex Subsets of Left-semimodules.}
Theorem~\ref{thm:S,etaS,muS weakly cartesian iff} together with Proposition~\ref{prop:(weak)distributive law with P iff weakly cartesian}.1 tell us that whenever the element $1$ of $S$ can be decomposed as a non-trivial sum there is no distributive law $\delta \colon \mon S \P \to \P \mon S$. Semirings with this property abound, for example $\mathbb Q$, $\mathbb R$, $\Rp$ with the usual operations of sum an multiplication, as well as $\Bool$ (since $1 \vee 1 = 1$). Such semirings are precisely those for which the notion of \emph{convex subset} of their left-semimodules is non-trivial. For the existence of a \emph{weak} distributive law, however, this condition on $1_S$ is not required: convexity will indeed play a crucial role in the definition of the weak distributive law.

\begin{definition}\label{def:convex closure SEMIMODULE}
	Let $S$ be a semiring, $X$ an $S$-left-semimodule and $A \subseteq X$. The \emph{convex closure} of $A$ is the set
	\[
	\convclos{A}{} =\left\{ \sum_{i=1}^n \lambda_i \cdot a_i \mid n \in \N,\, a_i \in A,\,  \sum_{i=1}^n \lambda_i =1 \right\} \subseteq X.
	\]
	The set $A$ is said to be \emph{convex} if and only if $A = \convclos A {}$.
\end{definition}

Recalling that the category of $S$-left-semimodules is isomorphic to $\EM{\mon S}$, we can use~\eqref{eq:S-algebra associated to semimodule} to translate Definition~\ref{def:convex closure SEMIMODULE} of convex subset of a semimodule into the following notion of convex subset of a $\mon S$-algebra $a\colon \mon S X \to X$.
\begin{definition}\label{def:convex closure S-ALGEBRA}
	Let $S$ be a semiring, $(X,a) \in \EM{\mon S}$, $A \subseteq X$. The \emph{convex closure} of $A$ in $(X,a)$ is the set
	\[
	\convclos A a = \left\{ a(\phi) \mid \phi \in \mon S X, \supp \phi \subseteq A, \sum_{x \in X} \phi(x) =1 \right\}.
	\]
	$A$ is said to be \emph{convex} in $(X,a)$ if and only if $A = \convclos A a$. We denote by $\ConvPow X a$ the set of convex subsets of $X$ with respect to $a$.
\end{definition}

\begin{remark}
	Observe that $\emptyset$ is convex, because $\convclos \emptyset a = \emptyset$, since there is no $\phi \in \mon S X$ with empty support such that $\sum_{x \in X} \phi(x) = 1$.
\end{remark}

\begin{example}\label{ex:convex subsets of N-semimodules}
Suppose $S$ is such that $\etaS$ is weakly cartesian (equivalently (A) holds: $x+y=1 \implies x=0$ or $y=0$), for example $S = \N$, and let $(X,a) \in \EM{\mon S}$. A $\phi \in \mon S X$ such that $\sum_{x \in X} \phi(x)=1$ and $\supp \phi \subseteq A$ is a function that assigns $1$ to \emph{exactly one} element of $A$ and $0$ to all the other elements of $X$. These functions are precisely all the $\Delta_x$ for those elements $x\in A$. Since $a\colon \mon S X \to X$ is a structure map for an $\mon S$-algebra, it maps the function $\Delta_x$ into $x$. Therefore $\convclos A a =\{a(\Delta_x) \mid x\in A\} = \{x \mid x \in A\} = A$. Thus \emph{all}  $A\in \pow \mon S X$ are convex.
\end{example}

\begin{example}\label{ex:S=Bool, convex subsets are directed}
	When $S = \Bool$, we have that $\S$ is naturally isomorphic to $\Pf$, the finite powerset monad, whose algebras are idempotent commutative monoids or equivalently semilattices with a bottom element. So, for $(X,a) \in \EM{\mon S}$, a $\phi \in \mon S X$ such that $\sum_{x \in X} \phi(x)=1$ and $\supp \phi \subseteq A$ is any finitely supported function from $X$ to $\Bool$ that assigns $1$ to at least one element of $A$. Intuitively, such a $\phi$ selects a non-empty finite subset of $A$, then  $a(\phi)$ takes the join of all the selected elements. Thus, $\convclos A a$ adds to $A$ all the possible joins of non-empty finite subsets of $A$: $A$ is convex if and only if it is closed under binary joins.
\end{example}

%% file: distributiveLaw.tex
\section{The Weak Distributive Law $\delta \colon \S \P \to \P \S$}\label{sec:the weak distributive law}

Weak extensions of $\mon S$ to $\Kl\P = \Rel$ only consist of extensions of the functor $\mon S$ and of the multiplication $\muS$, for which necessary and sufficient conditions are listed in Theorem~\ref{thm:S,etaS,muS weakly cartesian iff}. Hence for semirings $S$ satisfying those criteria a weak distributive law $\delta \colon \mon S \P \to \P \mon S$ does exist, and it is unique because there is only one extension of the functor $\mon S$ to $\Rel$.

\begin{theorem}\label{thm:existence of delta} 
	Let $S$ be a positive, refinable semiring satisfying \emph{(B)} and \emph{(E)} in Table~\ref{tab:semiring properties}. Then there exists a unique weak distributive law $\delta \colon \mon S \P \to \P \mon S$ defined for all sets $X$ and $\Phi \in \mon S \pow X$ as:
		\begin{equation}\label{eqn:def of delta}
	\delta_X (\Phi) = \Biggl\{ \phi \in \mon SX \mid \exists \psi \in \mon S (\ni_X) \ldotp \begin{cases}
	\forall A \in \pow X \ldotp \Phi(A) = \sum\limits_{x \in A} \psi(A,x) & (a)\\
	\forall x \in X \ldotp \phi(x) = \sum\limits_{A \ni x} \psi(A,x) & (b)
	\end{cases}  
	\Biggr\}
	\end{equation}
	where ${}\ni_X{}$ is the set $\{ (A,x) \in \pow X \times X \mid x \in A \}$.
\end{theorem} 
The above $\delta$, which is obtained by following the standard recipe of Proposition~\ref{prop:(weak)distributive law with P iff weakly cartesian} (see the proof in Appendix~\ref{Appendix Section 4}), 
is illustrated by the following example.

\begin{example}\label{ex:working example delta=delta' part 1}
Take $S=\Rp$ with the usual operations of sum and multiplication. Consider $X=\{x,y,z,a,b\}$, $A_1 =\{x,y\}$, $A_2=\{y,z\}$ and $A_3=\{a,b\}$. Let $\Phi \in \mon S (\pow X)$ be defined as
	\[\Phi = ( A_1 \mapsto 5, \quad A_2 \mapsto 9, \quad A_3 \mapsto 13)\]
	and $\Phi(A)=0$ for all other sets $A\subseteq X$, so $\supp \Phi = \{A_1,A_2,A_3\}$. In order to find an element $\phi \in \delta_X(\Phi)$, we can first take a $\psi\in \mon S (\ni_X)$ satisfying  condition (a) in \eqref{eqn:def of delta} and then compute the $\phi\in \mon S X$ using condition (b).
	
	Among the $\psi \in \mon S (\ni_X)$, consider for instance the following:
	\[
	\psi = \left(  
	\begin{array}{rclrclrcl}
	     (A_1,x)&\mapsto& 2 \quad & (A_2,y)&\mapsto& 4 \quad (A_3,a)&\mapsto& 6  \\
	     (A_1,y)&\mapsto& 3 \quad & (A_2,z)&\mapsto& 5 \quad (A_3,b)&\mapsto& 7  & 
	\end{array}
	\right).
	\]
Since $\Phi (A_1) = \psi(A_1,x) + \psi(A_1,y)$, $\Phi(A_2)=\psi(A_2,y) + \psi(A_2,z)$ and $\Phi(A_3)=\psi(A_3,a) + \psi(A_3, b)$, we have that $\psi$ satisfies condition (a) in \eqref{eqn:def of delta}. Condition (b) forces $\phi$ to be the following:
	\[\phi = ( x \mapsto 2, \quad y \mapsto 3+4, \quad z \mapsto 5, \quad a \mapsto 6, \quad b\mapsto 7).\]
\end{example}

\begin{remark}
	If $S$ enjoys (A) in Table~\ref{tab:semiring properties}, then the transformation $\delta$ given in~(\ref{eqn:def of delta}) is actually a distributive law, and for $S=\N$ we recover the well-known $\delta$ of Example~\ref{example:distributive law for N}. Example~\ref{ex:working example delta=delta' part 1} can be repeated with $S=\N$: then $\Phi$ is the multiset where the set $A_1$ occurs five times, $A_2$ nine times and $A_3$ thirteen times. The elements of $\delta_X(\Phi)$ are all those multisets containing one element per copy of $A_1$, $A_2$ and $A_3$ in $\supp\Phi$. The $\phi$ provided indeed contains five elements of $A_1$ (two copies of $x$ and three of $y$), nine elements of $A_2$ (four copies of $y$ and five of $z$), thirteen elements of $A_3$ (six copies of $a$ and seven of $b$).
\end{remark}

As Example~\ref{ex:working example delta=delta' part 1} shows, each element $\phi$ of $\delta_X(\Phi)$ is determined by a function $\psi$ choosing for each set $A \in \supp \Phi$ a finite number of elements $x^A_1,\dots,x^A_m$ in $A$ and $s^A_1,\dots,s^A_m$ in $S$ in such a way that $\sum_{j=1}^m s^A_j = \Phi(A)$. The function $\phi$ maps each $x^A_j$ to $s^A_j$ if the sets in $\supp \Phi$ are \emph{disjoint}; if however there are $x^A_{j}$ and $x^B_k$ such that $x^A_j=x^B_k$ (like $y$ in Example~\ref{ex:working example delta=delta' part 1}), then $x^A_j$ is mapped to $s^A_j + s^B_k$. 

Among those $\psi$'s, there are some special, \emph{minimal} ones as it were, that choose for each $A$ in $\supp \Phi$ exactly \emph{one} element of $A$, and assign to it $\Phi(A)$. The induced $\phi$ in $\delta_X(\Phi)$ can be described as $\sum_{A\in u^{-1}\{x\}} \Phi(A)$ (equivalently $\mon{S}(u)(\Phi)$\footnote{More precisely, we should write $\S(u)(\Phi')$ where $\Phi'$ is the restriction of $\Phi$ to $\supp \Phi$.}) where $u\colon \supp \Phi \to X$ is a function selecting an element of $A$ for each $A\in \supp \Phi$ (that is $u(A)\in A$). We denote the set of such $\phi$'s by $\choice \Phi$.
\begin{equation}\label{eqn:c(Phi)}
\choice \Phi =\{ \mon{S}(u)(\Phi) \mid u \colon \supp \Phi \to X \text{ such that } \forall A \in \supp \Phi \ldotp u(A)\in A \}
\end{equation}
\begin{example}
Take $X$, $A_1$ and $A_2$ as in Example~\ref{ex:working example delta=delta' part 1}, but a different, smaller, $\Phi\in  \mon S (\pow X)$ defined as
$\Phi = ( A_1 \mapsto 1, \quad A_2 \mapsto 2)$.
There are only four functions $u\colon \supp \Phi \to X$ such that $u(A)\in A$ and thus only four  functions $\phi$ in $\choice \Phi$:
\[\begin{array}{c|l}
     u_1=(A_1 \mapsto x, \quad A_2 \mapsto y) \quad & \quad  \phi_1 = (x \mapsto 1, \; y \mapsto 2)   \\
     u_2=(A_1 \mapsto x, \quad A_2 \mapsto z) \quad & \quad  \phi_2 = (x \mapsto 1, \; z \mapsto 2)   \\
     u_3=(A_1 \mapsto y, \quad A_2 \mapsto y) \quad & \quad  \phi_3 = ( y \mapsto 3)   \\
     u_4=(A_1 \mapsto y, \quad A_2 \mapsto z) \quad & \quad  \phi_4 = (y \mapsto 1, \; z \mapsto 2)  
\end{array}\]
Observe that the function $\phi= (x \mapsto 1, y\mapsto 1, z\mapsto 1)$ belongs to $\delta_X(\Phi)$ but not to $\choice \Phi$. Nevertheless $\phi$ can be retrieved as the convex combination $\frac{1}{2}\cdot \phi_1 + \frac{1}{2} \cdot \phi_2$.
\end{example}

Our key result states that every $\phi \in \delta_X(\Phi)$ can be written as a convex combination (performed in the $\S$-algebra $(\S X,\muS_X)$) of functions in $\choice\Phi$, at least when $S$ is a positive semifield, which by Remark~\ref{rem:S positive refinable semifield then ok} and Proposition~\ref{prop:positive semifield implies refinable} satisfies all the conditions that make (\ref{eqn:def of delta}) a weak distributive law. The proof is laborious and can be found in Appendix~\ref{Appendix Section 4};
we only remark that divisions in $S$ play a crucial role in it.
\begin{theorem}\label{thm:delta for positive refinable semifields}
	Let $S$ be a positive semifield. Then for all sets $X$ and $\Phi \in \S \P X$
	\begin{equation}\label{eqn:delta for positive refinable semifields}
	\delta_X(\Phi)=\left\{ \muS_X(\Psi) \mid \Psi \in \mon S^2 X\ldotp \sum\limits_{\phi \in \mon S X} \Psi(\phi)=1, \, \supp \Psi \subseteq \choice\Phi \right\} = \convclos{\choice \Phi}{\muS_X}.
	\end{equation}
\end{theorem}

\begin{remark}
	If we drop the hypothesis of semifield and only have the minimal assumptions of Theorem~\ref{thm:existence of delta}, then (\ref{eqn:delta for positive refinable semifields}) does not hold any more: $S=\N$ is a counterexample. Indeed, in this case every subset of $\S X$ is convex with respect to $\muS_X$ (see Example~\ref{ex:convex subsets of N-semimodules}), therefore we would have $\delta_X(\Phi)=\choice\Phi$, which is false: the function $\phi$ of Example~\ref{ex:working example delta=delta' part 1} is an example of an element in $\delta_X(\Phi) \setminus \choice\Phi$.
	\end{remark}

\newcommand{\macroA}{\mathcal{A}}
\begin{remark}\label{remarkKlinRot}
When $S = \Bool$ (which is a positive semifield), the monad $\S$ coincides with the monad $\P_f$. The function $\choice{\cdot}$ in~\eqref{eqn:c(Phi)} can then be described as
\begin{equation*}
\choice \macroA =\{ \P_f(u)(\macroA) \mid u \colon \macroA \to X \text{ such that } \forall A \in \macroA \ldotp u(A)\in A \}
\end{equation*}
for all $\macroA \in \P_f \P X $. It is worth remarking that this is the transformation $\chi$ appearing in Example 9 of~\cite{DBLP:conf/fossacs/KlinR15} (which is in turn equivalent to the one in Example 2.4.7 of~\cite{manes2007monad}). This transformation was erroneously supposed to be a distributive law, as it fails to be natural (see~\cite{DBLP:journals/entcs/KlinS18}). However, by taking its convex closure, as displayed in~\eqref{eqn:delta for positive refinable semifields}, one can turn it into a \emph{weak} distributive law.
\end{remark}

%% file: weakLifting.tex
\section{The Weak Lifting of $\P$ to $\EM\S$}\label{sec: the weak lifting}

By exploiting the characterisation of the weak distributive law $\delta$ (Theorem~\ref{thm:delta for positive refinable semifields}), we can now describe the weak lifting of $\P$ to $\EM\S$ generated by $\delta$.

Recall from Definition~\ref{def:convex closure S-ALGEBRA} that $\ConvPow X a$ is the set of convex subsets of $X$ with respect to the $\mon S$-algebra $a\colon \mon S X\to X$. The functions $\iota_{(X,a)} \colon	\ConvPow X a \to \pow X$
and $\pi_{(X,a)} \colon \pow X \to \ConvPow X a$ are defined for all $A\in \ConvPow X a$ and $B \in \pow X$ as
\begin{equation}\label{eq:iotapi}
\iota_{(X,a)}(A)=A    \qquad \text{and} \qquad \pi_{(X,a)}(B)=\convclos B a \text{,}
\end{equation}
that is $\iota_{(X,a)}$ is just the obvious set inclusion and 
$\pi_{(X,a)}$ performs the convex closure in $a$. The function $\alpha_a \colon \mon{S} \ConvPow X a \to \ConvPow X a$ is defined for all $\Phi \in \mon{S} \ConvPow X a$ as
\begin{equation}\label{eq:alphaa}
	\alpha_a(\Phi)= \{a(\phi) \; \mid \; \phi \in \choice{\Phi} \}\text{.}
\end{equation}
To be completely formal, above we should have written $\choice{\mon S (\iota) (\Phi)}$ in place of $\choice{\Phi}$, but it is immediate to see that the two sets coincide. Proving that $\alpha_a \colon \mon{S} \ConvPow X a \to \ConvPow X a$ is well defined (namely, $\alpha_a(\Phi)$ is a convex set) and forms an $\mon S$-algebra requires some ingenuity and will be shown later in Section~\ref{sec:proof}. 
The assignment $(X,a) \mapsto (\ConvPow X a, \alpha_a)$ gives rise to a functor $\tilde \P \colon \EM{\mon S} \to \EM{\mon S}$ defined on morphisms $f \colon (X,a) \to (X',a')$  as

\begin{equation}\label{eq:f}
    \tilde \P(f)(A) =  \pow f (A)
\end{equation}
for all $A\in \ConvPow X a$. For all $(X,a)$ in $\EM{\mon S}$,   $\eta^{\tilde \P}_{(X,a)} \colon (X,a) \to \tilde \P (X,a) $ and  $\mu^{\tilde \P}_{(X,a)} \colon \tilde \P \tilde \P (X,a) \to \tilde \P (X,a)$ are defined for $x\in X$ and $\mathcal A \in \ConvPow {(\ConvPow X a)} {\alpha_a}$ as
\begin{equation}\label{eq:etamu}
\eta^{\tilde \P}_{(X,a)}(x) =  \{x\} \qquad \text{and} \qquad \mu^{\tilde \P}_{(X,a)} (\mathcal A) = \bigcup_{A \in \mathcal A} A\text{.} 
\end{equation}

\begin{theorem}\label{thm:weaklift}
	Let $\S$ be a positive semifield. Then the canonical weak lifting of the powerset monad $\P$ to $\EM{\mon S}$, determined by~\eqref{eqn:delta for positive refinable semifields}, consists of the monad $(\tilde \P,\eta^{\tilde \P},\mu^{\tilde \P})$ on $\EM{\mon S}$ defined as in \eqref{eq:alphaa}, \eqref{eq:f}, \eqref{eq:etamu} and the natural transformations $\iota \colon \U{\mon S}\tilde \P \to \P \U{\mon S}$ and $\pi \colon \P \U{\mon S} \to \U{\mon S} \tilde \P$ defined as in \eqref{eq:iotapi}.
\end{theorem}

It is worth spelling out the left-semimodule structure on $\ConvPow X a$ corresponding to the $\mon S$-algebra $\alpha_a \colon \mon{S} \ConvPow X a \to \ConvPow X a$. Let us start with $\lambda \cdot^{\alpha_a} A $ for some $A\in \ConvPow X a$. By~\eqref{eqn:semimodule associated to S-algebra},  $\lambda \cdot^{\alpha_a} A  = \alpha_a(\Phi)$ where $\Phi=(A \mapsto \lambda)$. By~\eqref{eq:alphaa}, $\alpha_a(\Phi)=\{a(\phi) \; \mid \; \phi \in \choice{\Phi} \}$. Following the definition of $\choice{\Phi}$ given in \eqref{eqn:c(Phi)}, one has to consider functions $u\colon \supp \Phi \to X$ such that $u(B)\in B$ for all $B\in \supp \Phi$: if $\lambda \neq 0$, then $\supp \Phi= \{A\}$ and thus, for each $x\in A$, there is exactly one function $u_x \colon \supp \Phi \to X$ mapping $A$ into $x$. It is immediate to see that $\mon S (u_{x}) (\Phi)$ is exactly the function $(x \mapsto \lambda)$ and thus $a(\mon S (u_{x}) (\Phi))$ is, by \eqref{eqn:semimodule associated to S-algebra}, $\lambda \cdot^a x$. Now if $\lambda=0$, then $\supp \Phi = \emptyset$, so there is \emph{exactly one} function $u\colon \supp \Phi \to X$ and $\mon S (u) (\Phi)$ is the function mapping all $x \in X$ into $0$ and thus, by \eqref{eqn:semimodule associated to S-algebra}, $a(\mon S (u) (\Phi)) = 0^a$. Summarising,
    \begin{equation}\label{eq:lambdaP}
\lambda \cdot^{\alpha_a} A = \begin{cases}
	\{\lambda \cdot^a x \, \mid \; x \in A\} & \text{if }\lambda \neq 0\\
	\{0^a\}  & \text{if }\lambda = 0 
	\end{cases} \end{equation}
Following similar lines of thoughts, one can check that 
    \begin{equation}\label{eq:sumP}
A+^{\alpha_a} B = \{x+^ay \; \mid \; x\in A, \; y\in B\} \qquad \text{and} \qquad 0^{\alpha_a}=\{0^a\}\text{.}\end{equation}

\begin{remark}\label{remarkJacobs} 
By comparing \eqref{eq:sumP} and \eqref{eq:lambdaP} with (4) and (5) in~\cite{jacobs_coalgebraic_2008}, it is immediate to see
that our monad $\tilde \P$ coincides with a slight variation of Jacobs's convex powerset monad $\mathcal C$, the only difference being that we do allow for $\emptyset$ to be in $\ConvPow X a$. Jacobs insisted on the necessity of $\mathcal C(X)$ to be the set of \emph{non-empty} convex subsets of $X$, because otherwise he was not able to define a semimodule structure on $\mathcal C (X)$ such that $0 \cdot \emptyset = \{0^a\}$. However, we do manage to do so, since by~\eqref{eq:lambdaP}, $0 \cdot A ={0^a}$ for all $A$ and in particular for $A=\emptyset$. At first sight, this may look like an ad-hoc solution, but this is not the case: it is intrinsic in the definition of the unique weak lifting of $\P$ to $\EM{\mon S}$, as stated by Theorem \ref{thm:weaklift} and shown next.
\end{remark}

\subsection{Proof of Theorem~\ref{thm:weaklift}}\label{sec:proof}
By Theorem~\ref{thm:bijective correspondence (weak) distributive laws and (weak) rest}, the weak distributive law~\eqref{eqn:def of delta} corresponds to a weak lifting $\tilde \P$ of $\P$ to $\EM\S$, which we are going to show coincides with the data of \eqref{eq:iotapi}-\eqref{eq:etamu}. 
The image along $\tilde \P$ of a $\mon S$-algebra $(X,a)$ will be a set $Y$ together with a structure map $\alpha_a$ that makes it a $\mon S$-algebra in turn. Garner~\cite[Proposition 13]{garner_vietoris_2020} gives us the recipe to build $Y$ and $\alpha_a$ appropriately. $Y$ is obtained by splitting the following idempotent in $\Set$:
\begin{equation}\label{eqn:idempotent e}
e_{(X,a)}=\begin{tikzcd}
\pow X \ar[r,"\etaS_{\pow X}"] & \mon S (\pow X) \ar[r,"\delta_X"] & \pow(\mon S X) \ar[r,"\pow a"] & \pow X
\end{tikzcd}
\end{equation}
as a composite $e_{(X,a)}=\iota_{(X,a)} \circ \pi_{(X,a)}$, where $\pi_{(X,a)}$ is the corestriction of $e_{(X,a)}$ to its image and $\iota_{(X,a)}$ is the set-inclusion of the image of $e_{(X,a)}$ into $\pow X$. In other words, $Y$ is 
the set of fixed points of $e_{(X,a)}$. $\alpha_a$ is obtained as the composite
\[
\alpha_a=\begin{tikzcd}
\mon S Y \ar[r,"\mon S \iota_{(X,a)}"] & \mon S \pow X \ar[r,"\delta_X"] & \pow \mon S X \ar[r,"\pow a"] & \pow X \ar[r,"\pi_{(X,a)}"] & Y.
\end{tikzcd}
\]

Let us, then, fix an $\S$-algebra $(X,a)$. % and drop the subscripts in $e_{(X,a)}$.  
Given $A \in \pow X$, we have $\etaS_{\pow X}(A)=\Delta_A \colon \pow X \to S$, the Dirac-function centred in $A$. The set $\delta_X(\etaS_{\pow X}(A))$ has a simple description, shown in the next Lemma, whose proof is in Appendix~\ref{Appendix Section 5}.

\begin{lemma}\label{lemma:delta of Dirac}
	For all $A \in \pow X$ 
	\[
	\delta_X(\etaS_{\pow X}(A)) = \left\{ \phi \in \mon S X \mid \supp \phi \subseteq A, \sum_{x \in X} \phi(x) = 1 \right\}.
	\]
\end{lemma}
The image along $A$ of the idempotent $e$ is therefore
\[
e(A)=\pow a (\delta_X(\etaS_{\pow X}(A))) = \left\{a(\phi) \mid \phi \in \mon S X, \supp \phi \subseteq A, \sum_{x \in X} \phi(x)=1 \right\} = \convclos A a.
\]
Hence the idempotent $e$ computes the convex closure of elements of $\pow X$ and its fixed points are precisely the convex subsets of $X$ with respect to the structure map $a$. Therefore, the carrier set of $\tilde \P (X,a)$ is precisely $\ConvPow X a$, the natural transformations $\pi$ and $\iota$ are, respectively, the convex closure operator and the set-inclusion of $\ConvPow X a$ into $\pow X$ as in \eqref{eq:iotapi}.

$\ConvPow X a$ is then equipped with a structure map $\alpha_a \colon \mon S \ConvPow X a \to \ConvPow X a$ given by
\[
\alpha_a=\begin{tikzcd}
\mon S \ConvPow X a \ar[r,"\mon S \iota_{(X,a)}"] & \mon S \pow X \ar[r,"\delta_X"] & \pow \mon S X \ar[r,"\pow a"] & \pow X \ar[r,"\pi_{(X,a)}"] & \ConvPow X a.
\end{tikzcd}
\]
Let us try to calculate $\alpha_a$: given $\Phi \colon \ConvPow X a \to S$ with finite support, we have that $\mon S ({\iota_{(X,a)}}) (\Phi)$ is just the extension of $\Phi$ to $\pow X$ which assigns $0$ to each non-convex subset of $X$. If we write $\iota$ instead of $\iota_{(X,a)}$ for short, we have

\begin{equation}\label{eqn:alpha_a}
\alpha_a(\Phi) = \convclos { \pow a (\delta_X (\mon S (\iota) (\Phi))) } {a}.
\end{equation}
Next, we can use the following technical result, whose proof is in Appendix~\ref{Appendix Section 5}. \begin{proposition}\label{cor:image along a of convex subset of S(X) is convex in X}
	Let $(X,a)$ be a $\mon S$-algebra. If $\mathcal A$ is a convex subset of $(\mon S X, \muS_X)$, then
	$\pow a (\mathcal A)$
	is convex in $(X,a)$.
\end{proposition}

%Suppose that $S$ is a positive, refinable semifield. 
Since $\delta_X(\Phi')$ is the convex closure of $\choice{\Phi'}$ in $(\S X, \muS_X)$ for every $\Phi' \in \S \P X$, by Proposition~\ref{cor:image along a of convex subset of S(X) is convex in X} we can avoid to perform the $a$-convex closure in~(\ref{eqn:alpha_a}). Therefore
\[
\alpha_a(\Phi)=\pow a (\delta_X (\mon S (\iota) (\Phi)))=\pow a \bigl(\convclos{\choice{\mon S (\iota) (\Phi)}}{\muS_X}\bigr).
\] 
In the next Proposition we show that also the $\muS_X$-convex closure is superfluous, due to the fact that $\Phi \in \S \ConvPow X a$ (and not simply $\S\P X$), thus obtaining~\eqref{eq:alphaa}.

\begin{proposition}\label{thm:convpow(X,a) is a S-algebra}
	Let $S$ be a positive semifield, $(X,a)$ a $\mon S$-algebra, $\Phi \in \S\ConvPow X a$. Then $\pow a (\delta_X (\mon S (\iota) (\Phi))) = \pow a (\choice{\mon S (\iota) (\Phi)})$.
\end{proposition}
\begin{proof}
	In this proof we shall simply write $\Phi$ instead of the more verbose ${\mon S (\iota) (\Phi)}$. We want to prove that 
	\begin{multline}\label{eqn:linear combination of convex subsets of X}
	\pow a \bigl(\delta_X(\Phi)\bigr) = \\
	\Bigl\{
	a(\psi) \mid \psi \in \mon S X \ldotp \exists u \colon \!\! \supp \Phi \to X \ldotp u(A) \!\in \! A ,\, \forall x \in X\ldotp \psi(x) = \!\!\sum_{\substack{A \in \supp \Phi \\ u(A)=x}} \!\!\!\Phi(A)
	\Bigr\}
	\end{multline}
	where we have, by Theorem~\ref{thm:delta for positive refinable semifields}, that
	\[
	\pow a \bigl(\delta_X( \Phi)\bigr) = \\
	\{
	a(\muS_X(\Psi)) \mid \Psi \in \mon S ^2 X, \sum_{\phi \in \mon S X} \Psi(\phi)=1, \supp \Psi \subseteq \choice{\Phi}
	\}.
	\]
	First of all, $\emptyset$ is \emph{not} a $\mon S$-algebra, because there is no map $\mon S (\emptyset) \to \emptyset$ given that $\mon S (\emptyset) = \{ \emptyset \colon \emptyset \to S\}$, hence $X \ne \emptyset$. Next, if $\Phi = \zero \colon \P X \to S$, namely the function constant to $0$, then $\choice{\Phi} = \{\zero \colon X \to S\}$  therefore one can easily see that the left-hand side of~\eqref{eqn:linear combination of convex subsets of X} is equal to $\{a(\zero \colon X \to S)  \}$. For the same reason, the right-hand side is also equal to $\{a(\zero \colon X \to S)\}$. Moreover, if $\Phi(\emptyset) \ne 0$, then there is no $u \colon \supp \Phi \to X$ such that $u(\emptyset) \in \emptyset$, 
	so $\choice{\Phi} = \emptyset$ and so is the left-hand side of~(\ref{eqn:linear combination of convex subsets of X}); for the same reason, also the right-hand side is empty. 
	
	Suppose then, for the rest of the proof, that $\Phi \ne 0$ and that $\Phi(\emptyset)=0$.
	
	For the right-to-left inclusion in~(\ref{eqn:linear combination of convex subsets of X}): given $\psi \in \choice\Phi$, consider $\Psi = \etaS_{\mon S X} (\psi) = \Delta_\psi \in \mon S ^2 X$. Then $\Psi$ clearly satisfies all the required properties and $\muS_X(\Psi)=\psi$.
	
	The left-to-right inclusion is more laborious. Let $\Psi \in \mon S ^2 X$ be such that $\sum_{\chi \in \S X} \Psi(\chi)=1$ and such that $\supp \Psi \subseteq \choice\Phi$, that is, for all $\phi \in \supp \Psi$ there is $u^\phi \colon \supp \Phi \to X$ such that $u^\phi(A) \in A$ for all $A \in \supp \Phi$ and 
		$\phi = \S(u^\phi)(\Phi)$. 
	 We have to show that $a(\mu(\Psi))= a(\psi)$ for some $\psi \in \mon S X$ of the form $\sum_{A \in \supp \Phi} \Phi(A) \cdot u(A)$ for some choice function $u \colon \supp \Phi \to X$. Notice that the given $\Psi$ is a convex linear combination of functions $\phi$'s in $\mon S X$ like the one we have to produce: the trick will be to exploit the fact that each $A \in \supp \Phi$ is convex. Here we shall only give a sketch of the proof; the detailed version can be found in Appendix~\ref{Appendix Section 5}. 
	Suppose $\supp \Phi = \{A_1,\dots,A_n\}$ and $\supp \Psi = \{\phi^1,\dots,\phi^m\}$. Call $u^j$ the choice function that generates $\phi^j$. 
	Then $\Psi$ is of this form:
	\[
	\Psi= \Bigg(   
	\underbrace{\left(
	\parbox{2.75cm}{$u^1(A_1) \mapsto \Phi(A_1)$ \\ \vdots \\ $u^1(A_n) \mapsto \Phi(A_n)$}
	\right)}_{\phi^1} \mapsto \Psi(\phi^1), \,  \dots, \,  
	\underbrace{\left(
	\parbox{2.75cm}{$u^m(A_1) \mapsto \Phi(A_1)$ \\ \vdots \\ $u^m(A_n) \mapsto \Phi(A_n)$}
	\right)}_{\phi^m}
	\mapsto \Psi(\phi^m)
	\Bigg)
	\]
	Define the following element of $\S^2 X$:
	\[
	\Psi'= \Bigg(   
	\underbrace{\left(
	\parbox{2.75cm}{$u^1(A_1) \mapsto \Psi(\phi^1)$ \\ \vdots \\ $u^m(A_1) \mapsto \Psi(\phi^m)$}
	\right)}_{\chi^1}
	\mapsto \Phi(A_1), \, \dots, \,
	\underbrace{\left(
	\parbox{2.8cm}{$u^1(A_n) \mapsto \Psi(\phi^1)$ \\ \vdots \\ $u^m(A_n) \mapsto \Psi(\phi^m)$}
	\right)}_{\chi^n}
	\mapsto \Phi(A_n)
	\Bigg)
	\]
	Observe that $u^1(A_i), \dots, u^m(A_i) \in A_i$ by definition, and $A_i$ is convex by assumption: since $\sum_{j=1}^m \Psi(\phi^j)=1$, we have that $a(\chi^i) \in A_i$. Set then $u({A_i}) = a(\chi^i)$ and define $\psi= \S (a) (\Psi')$: we have $\psi \in \choice\Phi$ with $u$ as the generating choice function. It is not difficult to see that $\muS_X(\Psi)=\muS_X(\Psi')$, therefore we have
	\[
	a(\psi)=a\bigl(\mon S (a)(\Psi')  \bigr) = a\bigl(\muS_X (\Psi')\bigr) = a\bigl(\muS_X(\Psi)\bigr)
	\]
	as desired. \qed
\end{proof}

The rest of the proof of Theorem~\ref{thm:weaklift}, concerning the action of $\tilde \P$ on morphisms and the unit and multiplication of the monad $\tilde\P$, consists in following the recipe provided by Garner~\cite{garner_vietoris_2020}; the details can be found in Appendix~\ref{Appendix Section 5}. 
\label{end of section 5}

%% file: composedMonad.tex
\section{The Composite Monad: an Algebraic Presentation}\label{sec:the composite monad}

We can now compose the two monads $\P$ and $\mon S$ by considering the monad arising from the composition of the following two adjunctions:
\[
\begin{tikzcd}[column sep=3em]
{\Set} \ar[r,bend left=30,"{\F{\mon S}}"{name=F1},pos=0.55] & {\EM{\mon S}} \ar[l,bend left=30,"\U{\mon S}"{name=U1},pos=0.45] \ar[r,bend left=30,"\F{\tilde \P}"{name=F2},pos=0.55] & \EM{\tilde \P} \ar[l,bend left=30,"\U{\tilde \P}"{name=U2},pos=0.5]
\ar[from=F1,to=U1,"\bot"{description},draw=none]
\ar[from=F2,to=U2,"\bot"{description},draw=none]
\end{tikzcd}
\]
Direct calculations show that the resulting endofunctor on $\Set$, which we call $\convpowS{}$,
maps a set $X$ and a function $f\colon X \to Y$ into, respectively, 
\begin{equation}\label{composite endo}
\convpowS{X} = \ConvPow{(\mon S X)} {\muS_X} \qquad \text{and} \qquad\convpowS(f)(\mathcal A) = \{\mon{S}(f)(\Phi) \mid \Phi \in \mathcal{A}\}\end{equation}
for all $\mathcal{A}\in \convpowS{X}$. For all sets $X$, $\eta^\convpowS_X \colon X \to \convpowS{X}$ and $\mu^{\convpowS{}}_X \colon \convpowS{}\convpowS{X} \to \convpowS{X}$ are defined as
\begin{equation}\label{composite unit mult}
    \eta^\convpowS_X(x)= \{ \Delta_x \} \qquad \text{and} \qquad \mu^{\convpowS}_X(\mathscr A )= \bigcup\limits_{\Omega \in \mathscr A} \alpha_{\muS_X} (\Omega)
\end{equation}
for all $x\in X$ and $\mathscr A \in \convpowS{}\convpowS{X}$.

\begin{theorem}\label{thm:composite monad}
	Let $\S$ be a positive semifield. Then the canonical weak distributive law $\delta \colon \mon S \P \to \P \mon S$ given in Theorem~\ref{thm:delta for positive refinable semifields} induces a monad $\convpowS$ on $\Set$ with endofunctor, unit and multiplication defined as in \eqref{composite endo} and \eqref{composite unit mult}.
\end{theorem}

Recall from Remark~\ref{remarkJacobs} that the monad $\mathcal C \colon \EM{\mon S} \to \EM{\mon S}$ from~\cite{jacobs_coalgebraic_2008} coincides with our lifting $\tilde \P$ modulo the absence of the empty set. The same happens for the composite monad, which is named $\mathcal{CM}$ in~\cite{jacobs_coalgebraic_2008}. 
The absence of $\emptyset$ in $\mathcal{CM}$ turns out to be rather problematic for Jacobs. Indeed, in order to use the standard framework of coalgebraic trace semantics~\cite{hasuo_generic_2006}, one would need the Kleisli category $\Kl{\mathcal{CM}}$ to be enriched over $\Cppo$, the category of $\omega$-complete partial orders with \emph{bottom} and continuous functions. $\Kl{\mathcal{CM}}$ is not $\Cppo$-enriched since there is no bottom element in $\mathcal{CM}(X)$. Instead, in $\convpowS{X}$ the bottom is exactly the empty set; moreover, $\Kl{\convpowS{}}$ enjoys the properties required by~\cite{hasuo_generic_2006}.

\begin{theorem}\label{thm:Kleisli is CPPO enriched}
The category $\Kl{\convpowS}$ is enriched over $\Cppo$ and satisfies the left-strictness condition: for all $f \colon X \to \convpowS Y$ and $Z$ set, $\bot_{Y,Z} \circ f = \bot_{X,Z}$.
\end{theorem}
It is immediate that every homset in $\Kl{\convpowS}$ carries a complete partial order. Showing that composition of arrows in $\Kl{\convpowS}$ preserves joins (of $\omega$-chains) requires more work: the proof, shown in Appendix~\ref{Appendix Section 6}, 
crucially relies on the algebraic theory presenting the monad $\convpowS$, illustrated next.

\subsubsection{An Algebraic Presentation.} Recall that an \emph{algebraic theory} is a pair $\mathcal T=(\Sigma, E)$ where $\Sigma$ is a \emph{signature}, whose elements are called \emph{operations}, to each of which is assigned a cardinal number called its \emph{arity}, while $E$ is a class of formal \emph{equations} between $\Sigma$-terms. An \emph{algebra} for the theory $\mathcal T$ is a set $A$ together with, for each operation $o$ of arity $\kappa$ in $\Sigma$, a function $o_A \colon A^\kappa \to A$ satisfying the equations of $E$. A \emph{homomorphism} of algebras is a function $f \colon A \to B$ respecting the operations of $\Sigma$ in their realisations in $A$ and $B$. Algebras and homomorphisms of an algebraic theory $\mathcal T$ form a category $\Alg(\mathcal T)$.

\begin{definition}
	Let $M$ be a monad on $\Set$, and $\mathcal T$ an algebraic theory. We say that $\mathcal T$ \emph{presents} $M$ if and only if $\EM{M}$ and $\Alg(\mathcal T)$ are isomorphic.
\end{definition}

Left $S$-semimodules are algebras for the theory $\LSM= (\Sigma_{\LSM}, E_{\LSM})$ where
	$\Sigma_{S\LSM} = \{ +, 0 \} \cup \{ \lambda \cdot {} \mid \lambda \in S \}$ and $E_{\LSM}$ is the set of axioms in Table~\ref{tab:axiomsinitial}. As already mentioned in Section~\ref{sec:monad},  left $S$-semimodules are exactly $\mon S$-algebras and morphisms of $S$-semimodules coincide with those of $\mon S$-algebras. Thus, the theory $\LSM$ presents the monad $\mon S$.

Similarly, semilattices are algebras for the theory $\SL=(\Sigma_{\SL}, E_{\SL})$ where $\Sigma_{\SL} = \{\sqcup, \bot\}$ and $E_{\SL}$ is the set of axioms in Table~\ref{tab:axiomsinitial}. It is well known that semilattices are algebras for the \emph{finite} powerset monad. Actually, this monad is presented by $\SL$. 
In order to present the full powerset monad $\P$ we need to take joins of arbitrary arity. A \emph{complete semilattice} is a set $X$ equipped with joins $\Sup_{x\in A} x$ for all--not necessarily finite--$A\subseteq X$. Formally the (infinitary) theory of \emph{complete semilattices} is given as $\CSL = (\Sigma_{\CSL}, E_{\CSL})$ where 
$\Sigma_{\CSL} = \{ \Sup_{I} \mid I \text{ set} \}$
and $E_{\CSL}$ is the set of axioms displayed in Table~\ref{table:otheraxioms} (for a detailed treatment of infinitary algebraic theories see, for example,~\cite{manes_algebraic_1976}).

	\begin{table}[t]
	\caption{The sets of axioms $E_{\CSL}$ for complete semilattices: the second axiom generalises the usual idempotency and commutativity properties of finitary $\sqcup$, while the third one generalises associativity and neutrality of $\Sup_\emptyset =\bot$.}\label{table:otheraxioms}
	\begin{center}
	\begin{tabular}{l}
	 $\Sup_{i \in \{ 0\}} x_i  = x_0$ \\[.5em]
	 $\Sup_{j \in J}  x_j  =  \Sup_{i \in I}  x_{f(i)}  \text{ for all } f \colon I \to J \text{ surjective}$  \\[.5em]
	$\Sup_{i \in I} x_i  = \Sup_{j \in J} \Sup_{i \in f^{-1}\{j\}}  x_i   \text{ for all } f \colon I \to J $
	\end{tabular}
	\end{center}
	\end{table}

%\medskip

We can now illustrate the theory $(\Sigma,E)$ presenting the composed monad $\convpowS$: the operations in $\Sigma$ are exactly those of complete semilattices and $S$-semimodules, while the axioms are those of complete semilattices and $S$-semi\-modules together with the set $E_{\D}$ of \emph{distributivity} axioms illustrated below. 
\begin{equation}\label{eq:axiomsD}
    \lambda \cdot \Sup_{i \in I} x_i = \Sup_{i \in I} \lambda \cdot x_i \quad \text{for $\lambda \ne 0$}, \qquad \Sup_{i \in I} x_i + \Sup_{j \in J} y_j = \Sup_{(i,j) \in I \times J} x_i + y_j
\end{equation}
In short, $\Sigma = \Sigma_{\CSL} \cup \Sigma_{\LSM} \text{ and } E= E_{\CSL} \cup E_{\LSM} \cup E_{\D}$. %\text{.}
\begin{theorem}\label{thm:pres}
	The monad $\convpowS$ is presented by the algebraic theory $(\Sigma,E)$.
\end{theorem}

The presentation crucially relies on the fact that $\convpowS$ is obtained by composing $\P$ and $\mon S$ via $\delta$. Indeed, we know from general results in~\cite{bohm_weak_2010,garner_vietoris_2020} that  $\convpowS$-algebras are in one to one correspondence with $\delta$-algebras~\cite{beck_distributive_1969}, namely triples $(X,a,b)$ such that $a \colon \mon S X \to X$ is a $\mon S$-algebra, $b\colon \P X \to X$ is a $\P$-algebra and the following diagram commutes.
\begin{equation}\label{eqn:distributivity pentagon for delta-algebras}
\begin{tikzcd}
\mon S \pow X \ar[rr,"\delta_X"] \ar[d,"\mon S b"'] & & \pow \mon S X \ar[d,"\pow a"] \\
\mon S X \ar[dr,"a"'] & & \pow X \ar[dl,"b"] \\
& X 
\end{tikzcd}
\end{equation}
The $\mon S$-algebra $a$ corresponds to a $S$-semimodule $(X,+,0, \lambda \cdot)$, the $\P$-algebra $b$ to a complete lattice $(X, \Sup_I)$ and the commutativity of diagram \eqref{eqn:distributivity pentagon for delta-algebras} expresses exactly the distributivity axioms in \eqref{eq:axiomsD}. The proof is given in Appendix~\ref{Appendix Section 6}.

\begin{example}\label{ex:intervals}
Let $S$ be $\Rp$ and let $[a,b]$ with $a,b\in \Rp$ denote the set $\{x\in \Rp \mid a\leq x \leq b\}$ and $[a,\infty)$ the set $\{x\in \Rp \mid a\leq x \}$.
For $1=\{x\}$, $\convpowS(1) = \{\emptyset\} \cup \{[a,b] \,\mid \, a,b\in \Rp\} \cup \{[a,+\infty) \, \mid \, a\in \Rp\}$. The $\convpowS$-algebra $\mu^{\convpowS}_1 \colon \convpowS \convpowS 1 \to \convpowS 1$ induces a $\delta$-algebra where the structure of complete lattice is given as\footnote{For the sake of brevity, we are ignoring the case where some $A_i=\emptyset$.}
\[\Sup_{i\in I}A_i = \begin{cases}
[\inf_{i\in I}, a_i, \sup_{i\in I} b_i] & \text{if, for all }i\in I, \; A_i=[a_i, b_i] \wedge \sup_{i\in I} b_i\in \Rp\\
	[\inf_{i\in I}a_i, \infty)  & \text{otherwise}
	\end{cases} \\
	\]
The $\Rp$-semimodule is as expected, e.g., $[a_1,b_1]+[a_2,b_2]=[a_1+a_2,b_1+b_2]$.
\end{example}

\subsubsection{Finite Joins and Finitely Generated Convex Sets.} 

We now consider the algebraic theory $(\Sigma', E')$ obtained by restricting $(\Sigma,E)$ to finitary joins. More precisely, we fix
\[
	\Sigma' = \Sigma_{\SL}\cup \Sigma_{\LSM} \qquad E'= E_{\SL} \cup E_{\LSM} \cup E_{\D'}
	\]
where $(\Sigma_{\SL}, E_{\SL})$ is the algebraic theory for semilatices, $(\Sigma_{\LSM},E_{\LSM})$ is the one for $S$-semimodules, and $E_{\D'}$ is the set of distributivity axioms illustrated in Table \ref{tab:axiomsinitial}.
Thanks to the characterisation provided by Theorem \ref{thm:pres}, we easily obtain a function translating $\Sigma'$-terms into convex subsets (the proof is in Appendix~\ref{Appendix Section 6}).
\begin{proposition}\label{prop:injectivemonadmap}
Let $T_{\Sigma',E'}(X)$ be the set of $\Sigma'$-terms with variables in $X$ quotiented by $E'$. Let $\bb{\cdot}_X \colon T_{\Sigma',E'}(X) \to \convpowS(X)$ be the function defined as
\[
\begin{array}{rcl}
\bb{x} &=& \{\Delta_x\} \text{ for }x\in X\\
\bb{0} &=& \{0^{\mu^{\mon S}}\} \\
\bb{\bot} &=& \emptyset \\
\end{array}\qquad
\begin{array}{rcl}
\bb{\lambda \cdot t} &=& \begin{cases}
	\{\lambda \cdot^{\mu^{\mon S}} f \, \mid \; f \in \bb{t}\} & \text{if }\lambda \neq 0\\
	\{0^{\mu^{\mon S}}\}  & \text{otherwise}
	\end{cases} \\
\bb{t_1 + t_2} &=& \{f_1 +^{\mu^{\mon S}} f_2 \, \mid \, f_1 \in \bb{t_1}, \; f_2 \in \bb{t_2} \} \\
\bb{t_1 \sqcup t_2} &=& \convclos{\bb{t_1}\cup \bb{t_2}}{{\mu^{\mon S}}} \\
\end{array}
\]
Let $\bb{\cdot} \colon T_{\Sigma',E'} \to \convpowS$ be the family $\{\bb{\cdot}_X\}_{X \in |\Set|}$.
Then $\bb{\cdot}\colon T_{\Sigma',E'} \to \convpowS$ is a map of monads and, moreover, each  $\bb{\cdot}_X \colon T_{\Sigma',E'}(X) \to \convpowS(X)$ is injective.
\end{proposition}

We say that a set $\mathcal A \in \convpowS(X)$ is \emph{finitely generated} if there exists a finite set $\mathcal B \subseteq \mon{S}(X) $ such that $\convclos{\mathcal B}{}= \mathcal A$. 
We write $\convpowfS(X)$ for the set of all $\mathcal A \in \convpowS(X)$ that are finitely generated.
The assignment $X \mapsto \convpowfS(X)$ gives rise to a monad $\convpowfS \colon \Set \to \Set$ where the action on functions, the unit and the multiplication are defined as for $\convpowS$. The reader can find a proof of this, as well as of the following Theorem, in Appendix~\ref{Appendix Section 6}.

\begin{theorem}\label{thm:convpowfS is presented by (Sigma', E')}
The monads $T_{\Sigma',E'}$ and $\convpowfS$ are isomorphic. Therefore $(\Sigma',E')$ is a presentation for the monad $\convpowfS$.
\end{theorem}

\begin{example}
Recall  $\convpowS(1)$ for $S=\Rp$ from Example~\ref{ex:intervals}.
By restricting to the finitely generated convex sets, one obtains  $\convpowfS(1) = \{\emptyset\} \cup \{[a,b] \,\mid \, a,b\in \Rp\} $, that is the sets of the form $[a,\infty)$ are not finitely generated. Table \ref{tab:defbb1} illustrates the isomorphism $\bb{\cdot}\colon T_{\Sigma',E'}(1) \to \convpowS(1)$. 
It is worth observing that every closed interval $[a,b]$
is denoted by a term in $T_{\Sigma',E'}(1)$ for $1=\{x\}$: indeed, $\bb{(a \cdot x) \sqcup (b\cdot x)}= [a,b]$.
For $2=\{x,y\}$, $\convpowfS(2)$ is the set containing all convex polygons: for instance the term $(r_1\cdot x + s_1\cdot y)\sqcup (r_2\cdot x + s_2\cdot y) \sqcup (r_3\cdot x + s_3\cdot y)$ denote a triangle with vertexes $(r_i,s_i)$. For $n=\{x_0,\dots x_{n-1}\}$, it is easy to see that $\convpowfS(n)$ contains all convex $n$-polytopes.
\end{example}

\begin{table}[t]
\caption{The inductive defintion of the function $\bb{\cdot}_1 \colon T_{\Sigma',E'}(1) \to \convpowS (1)$ for $1=\{x\}$.}
\[
\begin{array}{rcl}
\bb{x} &=& [1,1]\\
\bb{0} &=& [0,0] \\
\bb{\bot} &=& \emptyset \\
\end{array}\qquad
\begin{array}{rcl}
\bb{\lambda \cdot t} &=& \begin{cases}
	[\lambda \cdot a, \lambda \cdot b ] & \text{if }\lambda \neq 0 ,\; \bb{t}=[a,b]\\
	\emptyset & \text{if }\lambda \neq 0 ,\; \bb{t}=\emptyset\\
	[0,0]  & \text{otherwise}
	\end{cases} \\
\bb{t_1 + t_2} &=& \begin{cases}
[a_1+a_2, b_1+b_2] & \text{ if } \bb{t_i}=[a_i,b_i]\\
\emptyset & \text{ otherwise}
\end{cases}\\
\bb{t_1 \sqcup t_2} &=& \begin{cases}
[min \;a_i, \, max \; b_i] & \text{ if } \bb{t_i}=[a_i,b_i]\\
[a_1,b_1] & \text{ if } \bb{t_1}=[a_1,b_1],\; \bb{t_2}=\emptyset \\
[a_2,b_2] & \text{ if } \bb{t_2}=[a_2,b_2],\; \bb{t_1}=\emptyset \\
\emptyset & \text{ otherwise}
\end{cases}  \\
\end{array}
\]
\label{tab:defbb1}
\end{table}

%% file: appendixSectionSix.tex
\section{Appendix for Section~\ref{sec:the composite monad}}\label{Appendix Section 6}

\subsubsection{The Presentation of $\convpowS$.} For $\delta \colon \mon S \P \to \P \mon S$ given in~\eqref{eqn:delta for positive refinable semifields}, we have that a $\delta$-algebra is a set $X$ together with a $\mon S$-algebra structure $a \colon \mon S X \to X$ and a $\P$-algebra-structure $b \colon \pow X \to X$ such that pentagon~\eqref{eqn:distributivity pentagon for delta-algebras} commutes. A morphism of $\delta$-algebras is a morphism of $\P$- and $\S$-algebra simultaneously by definition. Hence, by defining $+^a$, $\lambda \cdot^a$ and $0^a$ as in~\eqref{eqn:semimodule associated to S-algebra}, we have that $X$ is a $S$-left-semimodule; moreover, if we define for each set $I$
\begin{equation}\label{eq:defSup}
\Sup_{i \in I}^b x_i = b(\{x_i \mid i \in I \})
\end{equation}
then $X$ is also a complete semilattice. This means that $X$ satisfies all the equations in $E_\CSL$ and $E_\LSM$; moreover, every morphism of $\delta$-algebras preserves all the operations in $\Sigma$. The commutativity of pentagon~\eqref{eqn:distributivity pentagon for delta-algebras} will instead ensure that the equations in $E_\D$ listed in~\eqref{eq:axiomsD} are satisfied, as we prove in the following. For $A \subseteq X$, we shall sometimes write $\Sup^b A$ for $\Sup_{x \in A}^b x$. 

\begin{proposition}\label{prop:sup A = sup convclos A in delta algebras}
    Let $(X,a \colon \mon S X \to X, b \colon \pow X \to X)$ be a $\delta$-algebra.
	For each $A \subseteq X$, 
	\[
	\Sup_{x \in A}^b x = \Sup_{x \in \convclos A a}^b x.
	\]	
%	where $e(A)$ is the convex closure of $A$ in $(X,a)$:
%	\[
%	e(A) = \{ a(f) \mid f \in \mon S X, \supp f \subseteq A, \sum_{x \in X} f(x) =1 \}.
%	\]
\end{proposition}
\begin{proof}
    
	Let $\Phi = \etaS_{\pow X}(A) = \Delta_A \in \mon S \pow X$. Then
	\[
	a\bigl( \mon S b (\Phi) \bigr) = a\bigl( \mon S b (\Delta_A) \bigr) = a\bigl( \Delta_{\Sup A} \bigr) = \Sup A
	\]
	while
	\[
	\delta_X(\Delta_A) = \{ \phi \in \mon S X \mid \supp \phi \subseteq A, \sum_{x \in X} \phi(x)=1 \}
	\]
	because of Lemma~\ref{lemma:delta of Dirac}, hence
	\[
	\pow a \bigl( \delta_X(\Delta_A) \bigr) = \convclos A a .
	\] 
	By the commutativity of diagram (\ref{eqn:distributivity pentagon for delta-algebras}), we have that $\Sup A = \Sup \convclos A a $. \qed
\end{proof}

We can then prove that every $\delta$-algebra satisfies the equations in $E_\D$.

\begin{theorem}\label{thm:distributivity property of delta-algebras}
	Let $(X,a \colon \mon S X \to X, b \colon \pow X \to X)$ be a $\delta$-algebra. Then for all $A,B \subseteq X$ and for all $\lambda \in S \setminus \{ 0_S\}$:
	\[
	\lambda \cdot \Sup_{x \in A}^b x = \Sup_{x \in A}^b \lambda \cdot x, \qquad \Sup_{x \in A}^b x +^a \Sup_{y \in B}^b y  = \Sup_{(x,y) \in A \times B}^b  x +^a  y.
	\]
\end{theorem}
\begin{proof} 
	Let
	$
	\Phi = ( 
	A \mapsto \lambda) \in \S\P X
	$.
	We prove that the images along the two legs of pentagon~\eqref{eqn:distributivity pentagon for delta-algebras} of $\Phi$ coincide with the two sides of the first equation in the statement. 
It holds that 
\begin{align*}
    a(\mon S(b)(\Phi)) &= a(x \mapsto \sum_{U \in b^{-1}\{x\}} \Phi(U) ) \\ %&  \text{(def. of $\mon S$)} \\
    &= a(b(A) \mapsto \lambda) \\ %& \text{(def. of $\Phi$)} \\
    &= a(\Sup_{x\in A}^b x \mapsto \lambda) & \text{Definition of $\Sup^b$ \eqref{eq:defSup}} \\
    &= \lambda \cdot^a \Sup_{x\in A}^b x & \text{Definition of $\lambda \cdot^a$ \eqref{eqn:semimodule associated to S-algebra}}
\end{align*}
% \[\begin{array}{rcll}
% 	a(\mon S(b)(\Phi)) &=& a(x \mapsto \sum_{U \in b^{-1}\{x\}} \Phi(U) ) & \quad \text{(def. of }\mon S \text{)}\\
% 	& = & a(b(A) \mapsto \lambda) & \quad \text{(def. of } \Phi \text{)}\\
% 	& = & a(\Sup_{x\in A}^b x \mapsto \lambda ) & \quad \eqref{eq:defSup}\\
% 	& = & \lambda \cdot^a \Sup_{x\in A}^b x  & \quad \eqref{eqn:semimodule associated to S-algebra}
% 	\end{array}\]

On the other hand we have that 
	\begin{align*}
		b\Bigl(\pow a \bigl(\delta_X (\Phi)\bigr)\Bigr) &= \Sup^b \pow a \bigl( \convclos {\choice \Phi}{\muS_X} \bigr) & \text{Theorem~\ref{thm:delta for positive refinable semifields}}\\
		&= \Sup^b \convclos {\pow a \bigl( \choice \Phi \bigr)} {a} & \text{Proposition~\ref{prop:image along a of convex closures}}\\
		&= \Sup^b \pow a \bigl( \choice \Phi \bigr) & \text{Proposition~\ref{prop:sup A = sup convclos A in delta algebras}}
	\end{align*}

Following the discussion after Theorem \ref{thm:weaklift}, $\choice{\Phi} = \{(x \mapsto \lambda) \mid x\in A \}$ if $\lambda \neq 0$. %, while  $\choice{\Phi} = \{\zero \}$ if $\lambda =0$.
Therefore, if $\lambda \neq 0$, then $\pow a \bigl( \choice \Phi \bigr) = 	\{a(x \mapsto \lambda)  \, \mid \; x \in A\}$ which by \eqref{eqn:semimodule associated to S-algebra} is exactly $\{ \lambda \cdot^a x  \, \mid \; x \in A\}$.
%If $\lambda = 0$, then $\pow a \bigl( \choice \Phi \bigr) = 	\{a(\zero) \}$ which by \eqref{eqn:semimodule associated to S-algebra} is exactly $\{0^a\}$.

Therefore 
\[	b\Bigl(\pow a \bigl(\delta_X (\Phi)\bigr)\Bigr) = 
%\begin{cases}
	\Sup^b_{x\in A} \lambda \cdot^a x. %& \text{if }\lambda \neq 0\\
%	\{0^a\}  & \text{if }\lambda = 0
%\end{cases}
\]
We then conclude by using the commutativity of diagram~(\ref{eqn:distributivity pentagon for delta-algebras}). Using the function $\Phi' = (A \mapsto 1,\, B \mapsto 1) \in \S\P X$ and a similar argument, one shows the second equation as well. \qed
\end{proof}

Vice versa, we want to prove that every algebra for the theory $(\Sigma,E)$ is also a $\delta$-algebra. To this end, let $(X,+,\lambda \cdot {},0_X,\Sup_I)$ be a $(\Sigma,E)$-algebra. Then $(X,+,\lambda \cdot {}, 0_X)$ is a $S$-left-semimodule, $(X,\Sup_I)$ is a complete sup-semilattice. % and, denoting by $\Sup A = \Sup_A (i)$ for all $A \subseteq X$ with $i  \colon A \hookrightarrow X$ being the set-inclusion function,
% \begin{gather}\label{eqn:distributivity property of (Sigma,E)-algebras}
% %\sum_{j=1}^n \lambda_j \cdot \Sup A_j = \Sup \{ \sum_{j=1}^n \lambda_j \cdot a_j \mid (a_j)_{j \in \natset n} \in \prod_{j=1}^n A_j \}
% \lambda  \cdot \Sup A + \mu \cdot \Sup B = \Sup \{ \lambda \cdot x + \mu \cdot y \mid x \in A,\, y \in B \} \\
% \lambda \cdot \Sup \emptyset = \Sup \emptyset.
% \end{gather}
By defining
\[
\begin{tikzcd}[row sep=0em]
\mon S X \ar[r,"a"] & X \\
f \ar[r,|->] & \sum\limits_{x \in \supp f} f(x) \cdot x
\end{tikzcd}
\qquad
\begin{tikzcd}[row sep=0em]
\pow X \ar[r,"b"] & X \\
A \ar[r,|->] & \Sup_{x \in A} x
\end{tikzcd}
\]
we have that $(X,a) \in \EM{\mon S}$ and $(X,b) \in \EM\P$. We now have to check that pentagon~(\ref{eqn:distributivity pentagon for delta-algebras}) commutes. Given $A \subseteq X$, we define its convex closure $\convclos A{}$ in the usual way:
\[
\convclos A{} = \{ \sum_{i=1}^n {p_i \cdot a_i} \mid n \in \N, (p_i)_{i \in \natset n} \in S^n_1, (a_i)_{i \in \natset n} \in A^n  \}
\]
where $S^n_1 = \{ (p_i)_{i \in \natset n} \in S^n \mid \sum_{i=1}^n p_i = 1 \}$. Also, we shall denote by $\Sup A$ the element $b(A) \in X$ for any $A \subseteq X$. Then pentagon~(\ref{eqn:distributivity pentagon for delta-algebras}) commutes if and only if for all $\Phi \in \mon S \pow X$:
\[
\sum_{A \in \supp \Phi} \!\! \Phi(A) \cdot \Sup A = \Sup \convclos{ \Bigl\{ \sum_{A \in \supp \Phi}\!\! \Phi(A) \cdot u(A) \mid u \colon \supp \Phi \to X \ldotp \forall A  \ldotp u(A) \in A \Bigr\} } {}
\]
where the left-hand side is $a(\S(b)(\Phi))$ and the right-hand side is $b(\P(a)(\delta_X(\Phi)))$.

In the following Lemma we prove that if $X$ is a $(\Sigma,E)$-algebra and $A \subseteq X$, then $\Sup A = \Sup \convclos A {}$. Using this fact and the distributivity property~\eqref{eq:axiomsD} we will have shown the commutativity of pentagon~(\ref{eqn:distributivity pentagon for delta-algebras}).

\begin{lemma}
	Let $(X,+,\lambda \cdot {},0_X,\Sup_I)$ be a $(\Sigma,E)$-algebra. Then for all $A \subseteq X$
	\[
	\Sup A = \Sup \convclos A {}.
	\]
\end{lemma}
\begin{proof}
	Let $n \in \N$. Define $S^n_1 = \{ (p_1,\dots,p_n) \in S^n \mid \sum_{i=1}^n p_i =1\}$. Let $(p_i)_{i \in \natset n} \in S^n_1$. Then
	\[
	\Sup A = 1 \cdot \Sup A = (\sum_{i=1}^n p_i) \cdot \Sup A = \sum_{i=1}^n \bigl( p_i \cdot \Sup A \bigr) = \Sup \Bigl\{ \sum_{i=1}^n p_i \cdot a_i \mid (a_i)_{i \in \natset n} \in A^n \Bigr\}
	\]
	because of the properties of $S$-semimodule and the distributivity~\eqref{eq:axiomsD}. 
	%\ref{eqn:distributivity property of (Sigma,E)-algebras}). 
	Let $\le$ be the partial order determined by the complete semilattice structure of $X$. We have that 
	\[
	\forall n \in \N, \forall (p_i) \in S^n_1, \forall (a_i) \in A^n \ldotp \sum_{i=1}^n p_i \cdot a_i \le \Sup \Bigl\{ \sum_{i=1}^n p_i \cdot b_i \mid (b_i)_{i \in \natset n} \in A^n \Bigr\} = \Sup A
	\]
	hence
	\[
	\Sup \convclos A {} = \Sup \Bigl\{ \sum_{i=1}^n p_i \cdot a_i \mid n \in \N,\, (p_i) \in S^n_1, \, (a_i)_{i \in \natset n} A^n \Bigr\} \le \Sup A
	\]
	while the other inequality is trivial because $A \subseteq \Bigl\{ \sum_{i=1}^n p_i \cdot a_i \mid n \in \N,\, (p_i) \in S^n_1, \, (a_i)_{i \in \natset n} \in A^n \Bigr\}$. \qed
\end{proof}

Finally, again a morphism of $(\Sigma,E)$-algebras is a morphism respecting all the operations of $\Sigma$, which means of $\Sigma_{\LSM}$ (thus is a morphism of $\EM{\mon S}$) and of $\Sigma_{\CSL}$ (thus is a morphism of $\EM{\P}$) at the same time. We have therefore proved the following theorem.

\begin{theorem}
	The category $\EM{\delta}$ of $\delta$-algebras is isomorphic to the category  $\Alg(\Sigma,E)$ of $(\Sigma,E)$-algebras.
\end{theorem}

Since $\EM{\delta}$ is canonically isomorphic to $\EM{\convpowS}$ (\cite{bohm_weak_2010,garner_vietoris_2020}), we have proved Theorem~\ref{thm:pres}.

\subsubsection{The Kleisli Category of $\convpowS$.} In this section we aim to prove, using the algebraic presentation of the monad $\convpowS$, that its Kleisli category satisfies the conditions required to perform a coalgebraic trace semantics, as stated in~\cite{hasuo_generic_2006}.

$(\convpowS X,\mu^{\convpowS}_X)$ is a $\convpowS$-algebra, hence $\convpowS X$ has a structure of semimodule and complete lattice where, via the canonical isomorphism $\EM{\convpowS} \to \EM{\delta}$, we have
\begin{gather*}
    \mathcal A + \mathcal B = %\mu^\convpowS_X ( \{ (\mathcal A \mapsto 1,\, \mathcal B \mapsto 1 ) \} ) = 
    \{ \phi + \psi \mid \phi \in \mathcal A,\, \psi \in \mathcal B \} \\
    \lambda \cdot \mathcal A = %\mu^\convpowS_X ( \{ ( \mathcal A \mapsto \lambda ) \} ) = 
    \{ \lambda \cdot \phi \mid \phi \in \mathcal A \} \\
    \Sup_{i \in I} \mathcal A_i = %\mu^\convpowS_X ( \{ \Phi \in \S\convpowS X \mid \sum_{\mathcal A \in \convpowS X} \Phi(\mathcal A) = 1,\, \supp \Phi \subseteq \{ \mathcal A_i \mid i \in I \} \}  ) = 
    \convclos { \bigcup_{i \in I} \mathcal A_i } {}
\end{gather*}
for all $\mathcal A, \mathcal B \in \convpowS X$ and $\lambda \in S$, $\lambda \ne 0$.

Consider now the Kleisli category $\Kl{\convpowS}$. Each homset $\Kl{\convpowS}(X,Y)$ of functions $f \colon X \to \convpowS(Y)$ is partially ordered point-wise and inherits the structure of complete semilattice from $\convpowS Y$, in particular its bottom element $\bot_{X,Y}$ is the constant function mapping $x$ to $\emptyset$ for all $x \in X$. Given a function $g \colon Y \to \convpowS Z$, its Kleisli extension $\Sharp g \colon \convpowS Y \to \convpowS Z$ is given by $\mu^\convpowS_Z \circ \convpowS g$, which for all $\mathcal A \in \convpowS Y$ computes:
\begin{align*}
\Sharp g (\mathcal A) &= \Sup_{\phi \in \mathcal A} \, \sum_{y \in \supp \phi} \phi(y) \cdot g(y) \\
&= \bigcup_{\phi \in \mathcal A} \Bigl\{ \sum_{y \in \supp \phi} \phi(y) \cdot \psi_y \mid \forall y \in \supp \phi \ldotp \psi_y \in g(y) \Bigr\}.
\end{align*}
Composition of a function $f \colon X \to \convpowS Y$ and $g \colon Y \to \convpowS Z$ is therefore given as
\[
(g \circ f) (x) = \Sharp g (f(x))= \Sup_{\phi \in f(x)} \sum_{y \in \supp \phi } \phi(y) \cdot g(y) \in \convpowS Z.
\]

\begin{theorem}
The category $\Kl{\convpowS}$ is enriched over the category of directed-complete partial orders and satisfies the left-strictness condition:
\[
\bot_{Y,Z} \circ f = \bot_{X,Z}
\]
for all $f \colon X \to \convpowS Y$ and $Z$ set.
\end{theorem}
\begin{proof}
Let $f \colon X \to \convpowS Y$ and $\{ g_i \mid i \in I \}$ a directed subset of $\Kl{\convpowS}(Y,Z)$. Then:
\begin{align*}
    \Bigl( (\Sup_{i \in I} g_i ) \circ f \Bigr) (x) &= \Sup_{\phi \in f(x)} \sum_{y \in \supp \phi} \phi(y) \cdot (\Sup_{i \in I} g_i) (y) \\
    &= \Sup_{\phi \in f(x)} \sum_{y \in \supp \phi} \phi(y) \cdot \Sup_{i \in I} (g_i(y)) \\
    &= \Sup_{\phi \in f(x)} \sum_{y \in \supp \phi} \Sup_{i \in I} \phi(y) \cdot g_i(y) \\
    &\overset{\ast}{=} \Sup_{\phi \in f(x)} \Sup_{(i_y) \in I^{\supp \phi}} \{ \sum_{y \in \supp \phi} \phi(y) \cdot \psi_{i_y,y} \mid \forall y  \ldotp \psi_{i_y,y} \in g_i(y) \}  \\
    &\overset{\dagger}{=} \Sup_{\phi \in f(x)} \Sup_{i \in I} \{ \sum_{y \in \supp \phi} \phi(y) \cdot \psi_{i,y} \mid \forall y  \ldotp \psi_{i,y} \in g_i(y) \} \\
    &= \Sup_{\phi \in f(x)} \Sup_{i \in I} \sum_{y \in \supp \phi} \phi(y) \cdot g_i(y) \\
    &= \Sup_{i \in I} \Sup_{\phi \in f(x)} \sum_{y \in \supp \phi} \phi(y) \cdot g_i(y) \\ 
    &= \Sup_{i \in I} (g_i \circ f) (x).
\end{align*}
Equation ($\ast$) holds because of the distributivity of addition over joins~\eqref{eq:axiomsD}, while $(\dagger)$ holds because the family $\{g_i \mid i \in I\}$ is directed.

Given, instead, an \emph{arbitrary} subset $\{f_i \mid i \in I\}$ of $\Kl{\convpowS} (X,Y)$ and a $g \colon Y \to \convpowS Z$, we have
\begin{align*}
    \Bigl( g \circ (\Sup_{i \in I} f_i) \Bigr) (x) &= \Sharp g \Bigl(\Sup_{i \in I} f_i(x)\Bigr) \\
    &\overset{\ddag}{=} \Sup_{i \in I} \Sharp g (f_i(x)) \\
    &= \Sup_{i \in I} (g \circ f) (x)
\end{align*}
where equation $(\ddag)$ holds because $\Sharp g$ is a morphism of $\convpowS$-algebras (as it is given by the universal property of the free algebra $\convpowS Y$), hence it preserves arbitrary suprema. Finally, 
\[
(\bot_{Y,Z} \circ f)(x) = \Sup_{\phi \in f(x)} \sum_{y \in \supp \phi} \phi(y) \cdot \emptyset = \Sup_{\phi \in f(x)} \sum_{y \in \supp \phi} \emptyset = \Sup_{\phi \in f(x)} \emptyset = \emptyset
\]
(notice that $\phi(y) \cdot \emptyset = \emptyset$ because $\phi(y) \ne 0$ when $y \in \supp \phi$). \qed
\end{proof}

\subsubsection{The Term Monad for $(\Sigma,E)$.} The algebraic theory $(\Sigma,E)$ described in Theorem~\ref{thm:pres} determines a monad on $\Set$ $T_{\Sigma,E}$ where, for any set $X$, $T_{\Sigma,E}(X)$ is the set of all $\Sigma$-terms with variables in $X$ quotiented by the equations in $E$. Recall that a $\Sigma$-term with variables in $X$ is defined inductively as:
\begin{itemize}
    \item every variable $x$ is a $\Sigma$-term,
    \item if $o$ is an operation in $\Sigma$ with arity $\kappa$ and $t_1,\dots,t_\kappa$ are $\Sigma$-terms, then $o(t_1,\dots,t_\kappa)$ is a $\Sigma$-term.
\end{itemize}
If $f \colon X \to Y$ is a function, $T_{\Sigma,E}(f) \colon T_{\Sigma,E}(X) \to T_{\Sigma,E} (Y)$ sends a term $t$ in $t[f(x) / x]$, where every variable $x$ is substituted by its image $f(x)$. The unit $\eta^T$ is simply defined as $\eta^T_X (x) = x$, while the multiplication is defined by induction as:
\[
\begin{tikzcd}[row sep=0em]
T_{\Sigma,E} \bigl( T_{\Sigma,E} (X) \bigr) \ar[r,"\mu^T_X"] & T_{\Sigma,E} (X) \\
t \in T_{\Sigma,E} \ar[r,|->] & t  \\
o_\kappa(t_1,\dots,t_\kappa) \ar[r,|->] & o_\kappa (\mu^T_X(t_1),\dots,\mu^T_X(t_\kappa))
\end{tikzcd}
\]
This construction is standard for \emph{finitary} algebraic theories, where every operation in $\Sigma$ has finite arity. The fact that it makes sense also for our case, where we have an operation for every cardinal, is ensured by the fact that our $(\Sigma,E)$ is \emph{tractable} in the sense of~\cite[Definition 1.5.44]{manes_algebraic_1976}, because we proved in Theorem~\ref{thm:pres} that it presents a monad on $\Set$, namely $\convpowS$. Tractability ensures that the \emph{class} of $\Sigma$-terms, once quotiented by $E$, is forced to be a set. Notice, moreover, that if we represent a $\Sigma$-term, as usual, as a tree whose nodes are operation symbols in $\Sigma$, which have each as many branches as their arity, and whose leaves are variables, then we end up with a tree with infinite branches but finite height. This helps in giving an intuition of why we can define functions $\phi \colon T_{\Sigma,E}(X) \to Y$ by induction on the complexity of terms.

Now, the category of Eilenberg-Moore algebras for the monad $T_{\Sigma,E}$ is, in fact, isomorphic to $\Alg{(\Sigma,E)}$, hence also to $\EM{\convpowS}$, via a functor $F$ such that
\[
\begin{tikzcd}
\EM{\convpowS} \ar[d,"\U {\convpowS}"] \ar[r,"F"] &       \EM{T_{\Sigma,E}}  \ar[d,"\U{T}"] \\
\Set \ar[r,"\id{\Set}"] & \Set
\end{tikzcd}
\]
commutes. This generates an isomorphism of monads $\phi \colon T_{\Sigma,E} \to \convpowS$ where, for all $X$ set, $\phi_X = F(\mu^{\convpowS}_X) \circ T_{\Sigma,E} (\eta^{\convpowS}_X)$, thanks to the following general result:

\begin{theorem}
Let $(S,\eta^S,\mu^S)$ and $(T,\eta^T,\mu^T)$ be monads on a category $\C$. Suppose $\EM S$ and $\EM T$ are isomorphic via a functor $F \colon \EM S \to \EM T$ such that $\U T F = \U S$. Then $T$ and $S$ are isomorphic as monads, that is, there is a natural isomorphism $\phi \colon T \to S$ such that
\[
\begin{tikzcd}
\id\C \ar[r,"\eta^T"] \ar[dr,"\eta^S"'] & T \ar[d,"\phi"] & T^2 \ar[l,"\mu^T"'] \ar[d,"\phi \ast \phi"] \\
& S & S^2 \ar[l,"\mu^S"]
\end{tikzcd}
\]
commutes, where $\phi \ast \phi = \phi S \circ T \phi = S\phi \circ \phi T$. Specifically, $\phi_X$ is given as the unique morphism of $T$-algebras (hence, a morphism in $\C$) granted by the universal property of $(TX,\mu^T_X)$ as the free $T$-algebra on $X$, induced by $\eta^S_X$:
\[
\begin{tikzcd}
TTX \ar[d,"\mu^T_X"'] \ar[r,"T(\phi_X)"] & TSX \ar[d,"F(\mu^S_X)"] \\
TX \ar[r,dotted,"\exists ! \phi_X"] & SX \\
X \ar[u,"\eta^T_X"] \ar[ur,"\eta^S_X"']
\end{tikzcd}
\qquad
\phi_X = F(\mu^S_X) \circ T(\eta^S_X).
\]
\end{theorem}

Direct calculations show that our $\phi_X \colon T_{\Sigma,E}(X) \to \convpowS(X) = \ConvPow {\S X} {\muS_X}$ acts as follows:
\begin{align*}
    \phi_X(x) &= \{ \Delta_x \} \quad \text{for $x \in X$} \\
    \phi_X(0) &= \{ 0 \colon X \to S\} \\
    \phi_X(t_1 + t_2) &= \phi_X(t_1) + \phi_X(t_2) \\
    \phi_X (\lambda \cdot t) &= \lambda \cdot \phi_X (t) \\
    %&= \begin{cases}
    %\{ \lambda \cdot f \mid f \in \phi_X(t) \} & \text{if } \lambda \ne 0 \\
    %\{ 0 \colon X \to S \} & \text{otherwise} 
    %\end{cases} \\
    \phi_X(\Sup_I \{t_i \mid i \in I \}) &= \convclos {\bigcup_{i \in I} \phi(t_i)} {\muS_X}
\end{align*}
where $\phi_X(t_1) + \phi_X(t_2)$ is the result of adding up in the $\S$-algebra $(\ConvPow {\S X} {\muS_X}, \alpha_{\muS_X})$ (see Theorem~\ref{thm:convpow(X,a) is a S-algebra}) seen as a $S$-semimodule the convex subsets $\phi_X(t_1)$ and $\phi_X(t_2)$ of $\S X$. Similarly $\lambda \cdot \phi_X (t)$ is the scalar-product in $\ConvPow {\S X} {\muS_X}$. Hence:
\begin{align*}
    \phi_X(t_1 + t_2) &= \Bigl\{ \muS_X \Bigl(  
    \begin{tikzcd}[row sep=0em,ampersand replacement=\&]
    f_1 \ar[r,|->] \& 1 \\
    f_2 \ar[r,|->] \& 1
    \end{tikzcd}
    \Bigr) \mid f_1 \in \phi_X (t_1), \, f_2 \in \phi_X (t_2) \Bigr\} \\
    \lambda \cdot \phi_X (t) &= \{ \lambda \cdot f \mid f \in \phi_X (t) \}.
\end{align*}

It is clear, then, that if we restrict the action of $\phi_X$ to all those terms involving only finite suprema, we obtain exactly the function $\bb{\cdot}_X$ of Proposition~\ref{prop:injectivemonadmap}. Since $\phi_X$ is a bijection, its restriction $\bb{\cdot}_X$ is injective. This proves Proposition~\ref{prop:injectivemonadmap}.

\subsubsection{The Restriction of $\convpowS$ to $\convpowfS$.} The aim of this section is to prove that if we restrict the image of the functor $\convpowS$ on a set $X$ to those convex subsets $\mathcal A \subseteq \S X$ that are finitely generated, then all the remaining structure of the monad $\convpowS$ still works with no adaptation. This means that we have to prove that:
\begin{itemize}
    \item for $f \colon X \to Y$, $\convpowS(f) \colon \convpowS(X) \to \convpowS(Y)$ restricts and corestricts to $\convpowfS (X) \to \convpowfS (Y)$,
    \item $\eta^\convpowS_X$ can be corestricted to $\convpowfS (X)$ (trivial),
    \item $\mu^\convpowS_X$ restricts and corestricts to $\convpowfS \convpowfS X \to \convpowfS X$.
\end{itemize}

We do so in the next few results. From now on, if $\mathcal B \subseteq \S X$, we shall simply write $\convclos {\mathcal B} {}$ in lieu of $\convclos {\mathcal B} {\muS_X}$.

\begin{proposition}
Let $f \colon X \to Y$, $\mathcal A \subseteq \S X$ such that $\mathcal A =  \convclos {\mathcal B} {} $ for some finite $\mathcal B \subseteq \mathcal A$. Then
\[
\{ \S f (\phi) \mid \phi \in \mathcal A  \} = \convclos { \{\S f (\psi) \mid \psi \in \mathcal B \}} {}.
\]
\end{proposition}
\begin{proof}
For the left to right inclusion: let $\phi \in \mathcal A = \convclos {\mathcal B} {}$. Then there exists $\Psi \in \S^2 X$ such that $\sum_{\chi \in \S X} \Psi(\chi) =1$, $\supp \Psi \subseteq \mathcal B$, $\phi = \muS_X(\Psi)$. We have to prove that there is a $\Psi' \in \S^2 Y$ such that $\sum_{\chi \in \S Y} \Psi'(\chi) = 1$, $\supp \Psi' \subseteq \{\S f (\psi) \mid \psi \in \mathcal B \}$, $\muS_Y(\Psi')=\phi$.

Now, because of the naturality of $\muS$, we have
\[
\S f (\phi) = \S f (\muS_X(\Psi)) = \muS_Y(\S^2 f (\Psi)).
\]
One can easily see that $\S^2 f (\Psi)$ works for our desired $\Psi'$.

Vice versa, let $\Psi' \in \S^2 Y$ be such that $\sum_{\chi \in \S Y} \Psi'(\chi) = 1$ and $\supp \Psi' \subseteq \{\S f (\psi) \mid \psi \in \mathcal B \}$. We have to show that there is $\phi \in \mathcal A$ such that $\muS_Y (\Psi') = \S f (\phi)$.

We have that $\Psi'$ is of the form
\[
\left(
\begin{tikzcd}[row sep=0em]
\S f (\psi_1) \ar[r,|->] & \Psi'(\S f (\psi_1)) \\
\vdots & \vdots \\
\S f (\psi_n) \ar[r,|->] & \Psi'(\S f (\psi_n)) 
\end{tikzcd}
\right)
\]
Then that means that $\Psi'=\S(\S f)(\Psi)$ where $\Psi$ is defined as
\[
\Psi = \left(
\begin{tikzcd}[row sep=0em]
\psi_1 \ar[r,|->] & \Psi(\S f (\psi_1)) \\
\vdots & \vdots \\
\psi_n \ar[r,|->] & \Psi(\S f (\psi_n))
\end{tikzcd}
\right) \in \S^2 X
\]
Then, again by naturality of $\muS$, we have
\[
\muS_Y(\Psi')=\muS_Y(\S^2 f (\Psi)) = \S f (\muS_X (\Psi))
\]
and $\phi=\muS_X (\Psi)$ is indeed in $\mathcal A$ because $\mathcal A = \convclos {\mathcal B} {}$.
\qed
\end{proof}

%%%%%%%%%%%%%%%%%%%%%%%%%%%%%%%%%%%%%%%%%%%%%%%%%%%%%%%%%%%%%%%%%%%%%%%%%%%%%%%%%%%%%%%%%%%%%%%%%%%

This tells us that $\convpowfS$ is an endofunctor on $\Set$. Next,$\eta^\convpowS_X (x) = \{ \Delta_x \}$ and $\{\Delta_x\}$ is obviously finitely generated, therefore $\eta^\convpowS$ corestricts to $\convpowfS$. How about $\mu^\convpowS$?

Recall that $\mu^\convpowS_X \colon \convpowS \convpowS X \to \convpowS X$ is defined, for every $\mathscr A$ convex subset of $\S \bigl( \convpow (\S X) \bigr)$, as
\[
\mu^{\convpowS}_X (\mathscr A) = \bigcup_{\Omega \in \mathscr A} \{ \muS_X (F) \mid F \in \choice \Omega \}
\]
where 
\[
\choice\Omega = \{ F \in \S^2 X \mid \forall \mathcal A \in \supp \Omega \ldotp \exists u_{\mathcal A} \in \mathcal A \ldotp \forall \phi \in \S X \ldotp F(\phi) = \sum_{\substack{\mathcal A \in \supp \Phi \\ \phi = u_{\mathcal A}}} \Omega(\mathcal A) \}
\]
We aim to prove that $\bigcup_{\Omega \in \mathscr A} \{ \muS_X (F) \mid F \in \choice \Omega \}$ is, in fact, finitely generated in the hypothesis that $\mathscr A \in \convpowfS \convpowfS X$. We will achieve this in three steps.
\begin{itemize}
    \item \textbf{Step 1}: let $\mathscr B$ be a finite subset of $\mathscr A$ such that $\mathscr A = \convclos {\mathscr B} {}$. Then we prove that
    \[
    \bigcup_{\Omega \in \mathscr A} \{ \muS_X (F) \mid F \in \choice \Omega \} = 
    \convclos { \bigcup_{\Theta \in \mathscr B} \{ \muS_X (G) \mid G \in \choice \Theta \} } {\muS_X}
    \]
    showing therefore that we can reduce ourselves to a finite union.
    \item \textbf{Step 2}: we prove that each $\{ \muS_X (G) \mid G \in \choice \Theta \}$ as of Step 1 is convex and finitely generated.
    \item \textbf{Step 3}: we prove that the convex closure of a finite union of convex and finitely generated sets is in turn finitely generated.
\end{itemize}
The next three Lemmas will perform each step.

\begin{lemma}
Let $\mathscr A \in \convpowfS \convpowfS X$ and let $\mathscr B$ be a finite subset of $\mathscr A$ such that $\mathscr A = \convclos {\mathscr B} {}$. Then
\[
    \bigcup_{\Omega \in \mathscr A} \{ \muS_X (F) \mid F \in \choice \Omega \} = 
    \convclos { \bigcup_{\Theta \in \mathscr B} \{ \muS_X (G) \mid G \in \choice \Theta \} } {\muS_X}.
    \]
\end{lemma}
\begin{proof}
Let $\Omega \in \mathscr A = \convclos{\mathscr A}{}$. We have that
\[
\Omega = \muS_{\convpowS X} \left(
\begin{tikzcd}[row sep=0em]
\Theta_1 \ar[r,|->] & \sigma_1 \\
\vdots & \vdots \\
\Theta_t \ar[r,|->] & \sigma_t
\end{tikzcd}
\right)
\]
where $\Theta_i \in \mathscr B$ and $\sum_{l=1}^t \sigma_l = 1$. Notice that $\supp \Omega = \bigcup_{l=1}^t \supp \Theta_l$. Now, if $\supp \Omega = \{ \mathcal A_1, \dots, \mathcal A_n \} $ say, we have that any $F \in \choice \Omega$ is of the form
\[
F =  \left(
\begin{tikzcd}[row sep=0em]
\phi_1 \ar[r,|->] & \Omega(\mathcal A_1)=\sum_{l=1}^t \sigma_l \cdot \Theta_l (\mathcal A_1) \\
\vdots & \vdots  \\
\phi_n \ar[r,|->] & \Omega(\mathcal A_n) = \sum_{l=1}^t \sigma_l \cdot \Theta_l (\mathcal A_n)
\end{tikzcd}
\right)
\]
where each $\phi_i \in \mathcal A_i \subseteq \S X$. This leads us to define, for each $l \in \natset t$, a function $G_l \in \S^2 X$ as:
\[
G_l = \left(
\begin{tikzcd}[row sep=0em]
\phi_1 \ar[r,|->] & \Theta_l (\mathcal A_1) \\
\vdots & \vdots \\
\phi_n \ar[r,|->] & \Theta_l (\mathcal A_n)
\end{tikzcd}
\right)
\]
Notice that, in fact, $G_l \in \choice{\Theta_l}$. Indeed, fixed $l$, we have that $\supp \Theta_l \subseteq \supp \Omega$, so it can be the case that $\Theta_l(\mathcal A_i)=0$ for some $i \in \natset n$, in which case we have $G_l(\phi_i) = 0$. Nonetheless, for each $\mathcal A_i$ in $\supp {\Theta_l}$, the function $G_l$ chooses one of its elements and associates to it $\Theta_l (\mathscr A_i)$.

Define then
\[
H=\left(
\begin{tikzcd}[row sep=0em]
\muS_X(G_1) \ar[r,|->] & \sigma_1 \\
\vdots & \vdots \\
\muS_X(G_t) \ar[r,|->] & \sigma_t
\end{tikzcd}
\right)
\]
Then clearly $\sum_{\chi \in \S X} H(\chi) = 1$, and 
\[
\supp \chi \subseteq \bigcup_{l=1}^t \{ \muS_X(G) \mid G \in \choice {\Theta_l} \} \subseteq \bigcup_{\Theta \in \mathscr B} \{ \muS_X (G) \mid G \in \choice \Theta \}.
\]
It is a matter of direct calculation to show that $\muS_X (F) = \muS_X (H)$. This proves the left-to-right inclusion. The vice versa is immediate to see, recalling that we know that the left-hand side is convex. \qed
\end{proof}
Applying the following Lemma for the $\S$-algebra $(\S X, \muS_X)$, we obtain Step 2.
\begin{lemma}
Let $(X,a)$ in $\EM \S$. Then for all $\Phi \in \S \P_{cf}^a X$ we have that
\[
\{ a(\phi) \mid \phi \in \choice\Phi \} = \convclos { \{ a(\psi) \mid \psi \in \choice{\Phi'} \} } a
\]
where, if $\supp \Phi = \{ A_1,\dots,A_n \}$ say, and for all $i$ $A_i = \convclos {B_i} a$ with $B_i \subseteq A_i$ finite, then
\[
\Phi' = (B_1 \mapsto \Phi(A_1), \dots, B_n \mapsto \Phi(A_n)).
\]
\end{lemma}
\begin{proof}
Let $\phi \in \choice \Phi$. Then we can write $\phi$ as
\[
\phi = (u_1 \mapsto \Phi(A_1), \dots, u_n \mapsto \Phi(A_n)
\]
where $u_i \in A_i$ for all $i$. Notice that it is possible for $u_i = u_j$ for $i \ne j$: in that case, the notation above implicitly says that $u_i \mapsto \Phi(A_i) = \Phi(A_j)$. Now, since $A_i = \convclos {B_i} a$, we have that $u_i = a (\chi_i)$ for some $\chi_i \in \S X$ such that $\sum \chi_i(x) = 1$ and $\supp \chi_i \subseteq B_i$. So:
\begin{align*}
    \phi &= (a(\chi_1) \mapsto \Phi(A_1), \dots, a(\chi_n) \mapsto \Phi(A_n)) \\
    &= \S (a) \bigl( \chi_1 \mapsto \Phi(A_1), \dots, \chi_n \mapsto \Phi(A_n) \bigr)
\end{align*}
Call $\Psi = \bigl( \chi_1 \mapsto \Phi(A_1), \dots, \chi_n \mapsto \Phi(A_n) \bigr)$. Then we just said that $\phi = \S (a) (\Psi)$. We can write $\Psi$ more explicitly, by listing down the action of each $\chi_i$:
\[
\Psi = \left(
\begin{tikzcd}[row sep=0em]
\chi_1=\left(
	\parbox{1.6cm}{$x^1_1 \mapsto \lambda^1_1$ \\ \vdots \\ $x^1_{s_1} \mapsto \lambda^1_{s_1}$}
	\right)
	\ar[r,|->] & \Phi(A_1) \\
	\vdots & \vdots \\
	\chi_n=\left(
	\parbox{1.6cm}{$x^n_1 \mapsto \lambda^n_1$ \\ \vdots \\ $x^n_{s_n} \mapsto \lambda^n_{s_n}$}
	\right)
	\ar[r,|->] & \Phi(A_n)
\end{tikzcd}
\right)
\]
where, for each $k=1,\dots,n$, $\sum_{j=1}^{s_k} \lambda^k_j = 1$ and for all $j=1,\dots,s_k$ we have $x^k_j \in B_k$.

Now, define for all $w \in \prod_{k=1}^n \natset{s_k}$, a $n$-tuple whose $j$-th entry is a number between $1$ and $s_j$, the function
\[
\psi_w = \bigl( x^1_{w_1} \mapsto \Phi(A_1), \dots, x^n_{w_n} \mapsto \Phi(A_n) \bigr), \text{ i.e. } \psi_w(x) = \sum_{\substack{i \in \natset n \\ x=x^i_{w_i}}} \Phi(A_i).
\]
Define also
\[
\Psi' = ( \psi_w \mapsto \prod_{k=1}^n \lambda^k_{w_k} )_{w \in \prod \natset {s_k}} \text{ i.e. } \Psi'(\psi) = \sum_{\substack{w \in \prod \natset{s_k} \\ \psi=\psi_w }} \prod_{k=1}^n \lambda^k_{w_k}.
\]
Notice that, by Lemma~\ref{lemma:generalised distributivity}, we have that
\[
\sum_{w \in \prod \natset{s_k}} \prod_{k=1}^n \lambda^k_{w_k} = \prod_{k=1}^n \sum_{j=1}^{s_k} \lambda^k_j = \prod_{k=1}^n 1 = 1.
\]
This immediately implies that $\sum_{\psi \in \S X} \Psi'(\psi)=1$. Next, we show that $\muS_X (\Psi') = \muS_X(\Psi)$.
\begin{align*}
    \muS_X(\Psi')(x) &= \sum_{\psi \in \S X} \Psi'(\psi) \cdot \psi(x) \\
    &= \sum_w \Bigl[ \prod_{k=1}^n \lambda^k_{w_k} \cdot \sum_{\substack{i \in \natset n \\ x=x^i_{w_i} }} \Phi(A_i)  \Bigr] \\
    &= \sum_w \sum_{\substack{i \in \natset n \\ x=x^i_{w_i}}} \Bigl[ (\prod_{k=1}^n \lambda^k_{w_k} ) \cdot \Phi(A_i)  \Bigr] \\
    &= \sum_{i \in \natset n} \sum_{\substack{ w \\ x=x^i_{w_i}}} \Bigl[ \Phi(A_i) \cdot \underbrace{ \lambda^i_{w_i}}_{=\chi_i(x^i_{w_i}) = \chi_i(x)} \cdot \prod_{\substack{k \in \natset n \\ k \ne i}} \lambda^k_{w_k} \Bigr] \\
    &= \sum_{i=1}^n \Phi(A_i) \cdot \chi_i(x) \cdot \underbrace{\sum_{\substack{w \\ x=x^i_{w_i}}} \prod_{k \ne i} \lambda^k_{w_k}}_{=\prod_{k \ne i} \sum_{j=1}^{s_k} \lambda^k_j = 1} \\
    &= \muS_X(\Psi)(x).
    \end{align*}
    Hence:
    \[
    a(\phi) = a \bigl( \S(a)(\Psi)  \bigr) = a(\muS_X(\Psi)) = a(\muS_X (\Psi')) = a (\S(a)(\Psi'))
    \]
    where $a(\S(a)(\Psi'))$ indeed belongs to $\convclos { \{ a(\psi) \mid \psi \in \choice{\Phi'} \} } a$ because 
    \[
    \supp{(\S(a)(\Psi'))} = \{ a(\psi_w) \mid w \in \prod_{k=1^n} \natset{s_k} \} \subseteq \{ a(\psi) \mid \psi \in \choice{\Phi'} \}
    \]
    and the sum of all its images is $\sum_w \prod_{k=1}^n \lambda^k_{w_k} = 1$. \qed
\end{proof}

Notice that if $\Phi' \in \S\P X$ is such that $\supp{\Phi'} \subseteq \P_f (X)$, then $\choice{\Phi'}$ is finite. Finally, the next Lemma, when $A$ and $B$ are finite, proves Step 3.

\begin{lemma}
Let $(X,a)$ be in $\EM \S$, $A,B \subseteq X$. Then $\convclos{\convclos A {} \cup \convclos B {} } {} = \convclos{A \cup B} {}$.
\end{lemma}
\begin{proof}
We have:
\begin{align*}
    \convclos { \convclos A {} \cup \convclos B {} } {} &= \{ a(\phi) \mid \phi \in \S X, \, \sum_{x} \phi(x)=1, \supp \phi \subseteq \convclos A {} \cup \convclos B {} \} \\
    \convclos A {} \cup \convclos B {} &= \{ a(\psi) \mid \psi \in \S X,\, \sum_x \psi(x)=1,\, \supp \psi \subseteq A \text{ or } \supp \psi \subseteq B \} \\
    \convclos{A \cup B} {} &= \{ a(\chi) \mid \chi \in \S X,\, \sum_x \chi(x) =1,\, \supp \chi \subseteq A \cup B \}.
\end{align*}
Let then $x \in \convclos{\convclos A {} \cup \convclos B {} } {}$. Then
\[
x=a(\phi) = a \left(
\begin{tikzcd}[row sep=0em]
a(\psi_1) \ar[r,|->] & \phi(a(\psi_1)) \\
\vdots \\
a(\psi_n) \ar[r,|->] & \phi(a(\psi_n))
\end{tikzcd}
\right)
= a(\S (a) (\Phi)) = a (\muS_X (\Phi))
\]
where $\Phi = (\psi_1 \mapsto \phi(a(\psi_1)),\dots, \psi_n \mapsto \phi(a(psi_n))$. Calling $\chi=\muS_X(\Phi)$, one can easily check that $\sum_x \chi(x)=1$ and that $\supp \chi \subseteq A \cup B$, hence $\convclos{\convclos A {} \cup \convclos B {} } {} \subseteq \convclos{A \cup B} {}$. The other inlusion is obvious, given that $A \cup B \subseteq \convclos A {} \cup \convclos B {}$. \qed
\end{proof}

\subsubsection{Proof of Theorem~\ref{thm:convpowfS is presented by (Sigma', E')}.} We first observe that the function 
\[
\bb{\cdot}_X \colon T_{\Sigma',E'}(X) \to \convpowS(X)
\]
factors as 
\[\begin{tikzcd}
T_{\Sigma',E'}(X) \ar[r,"\bb{\cdot}'_X"] &  \convpowfS(X) \ar[r,"\iota_X"] & \convpowS(X)
\end{tikzcd}\]
where $\iota_X \colon \convpowfS(X) \to\convpowS(X)$ is the obvious set-inclusion. This can be easily checked by induction on $T_{\Sigma',E'}(X)$.

Observe that, since $\bb{\cdot}_X$ is injective by Proposition \ref{prop:injectivemonadmap}, then also $\bb{\cdot}'_X$ %\colon T_{\Sigma',E'}(X) \to \convpowfS(X)$ 
is injective. We conclude by showing that it is also surjective.

Let $\mathcal A \in \convpowfS(X)$. Since $\mathcal A$ is finitely generated there exists a finite set $\mathcal B \subseteq \mon{S}(X)$ such that $\convclos{\mathcal B}{}=\mathcal A$. If $\mathcal A=\emptyset$, then $\bb{\bot}'_X = \mathcal A$. If $\mathcal B=\{\phi_1, \dots, \phi_n\}$ with $\phi_i \in \mon{S}(X)$ then, for all $i$, we take the term 
$$t_i=\phi_i(x_1)\cdot x_1 + \dots + \phi_i(x_m)\cdot x_m$$ 
where $\{x_1, \dots, x_m\}$ is the support of $\phi_i$.
It is easy to check that $\bb{t_i}_X'= \{\phi_i\}$. Then by the inductive definition of $\bb{\cdot}'_X$, one can easily verify that $\bb{t_1 \sqcup \dots \sqcup t_n}'_X = \convclos{\{\phi_1, \dots \phi_n\}}{}= \mathcal A$. \qed